\documentclass{article}

\usepackage{arxiv}

\usepackage[utf8]{inputenc} 
\usepackage[T1]{fontenc}    
\usepackage{hyperref}       
\usepackage{url}            
\usepackage{booktabs}       
\usepackage{amsfonts}       
\usepackage{nicefrac}       
\usepackage{microtype}      
\usepackage{lipsum}         
\usepackage{graphicx}
\usepackage{natbib}
\usepackage{doi}

\usepackage{hyperref}
\usepackage{url}
\usepackage{graphicx,subcaption}   
\usepackage{adjustbox}
\usepackage{multirow}
\usepackage{booktabs} 
\usepackage{makecell}
\usepackage{amsmath,amssymb,amsfonts}
\usepackage{amsthm}
\usepackage{cleveref}       
\newtheorem{assumption}{Assumption}

\theoremstyle{plain}

\newtheorem{proposition}{Proposition}   %

\theoremstyle{definition}
\theoremstyle{remark}

\usepackage{xcolor}
\usepackage{booktabs}
\usepackage{listings}
\usepackage[most]{tcolorbox}
\tcbuselibrary{skins,breakable}

\title{Self-Evolving Algorithm-Design Agents: Escaping In-Context Evolutionary Stagnation via Population-Curated Policy Optimization}

\date{}

\author{ 
Chen Lu$^{1,2}$, Ke Xue$^{1,2}$, Siyuan Xu$^{3}$, Mingxuan Yuan$^{3}$, Chao Qian$^{1,2}$
\thanks{Corresponding Author: qianc@nju.edu.cn} \\
  $^1$ State Key Laboratory of Novel Software Technology, Nanjing University, China\\
  $^2$ School of Artificial
Intelligence, Nanjing University, China\\
  $^3$ Huawei Noah’s Ark Lab, China\\
}

\renewcommand{\shorttitle}{\textit{arXiv Preprint}}

\hypersetup{
pdftitle={A template for the arxiv style},
pdfsubject={q-bio.NC, q-bio.QM},
pdfauthor={David S.~Hippocampus, Elias D.~Striatum},
pdfkeywords={First keyword, Second keyword, More},
}

\begin{document}
\maketitle

\begin{abstract}
	Large language models are increasingly participating in complex real-world tasks in the form of algorithm-design agents, designing and refining algorithms. Many successful algorithm-design agents adopt pure in-context evolutionary frameworks, but they may quickly plateau in domains that require specialized knowledge. Parametric adaptation offers a way to internalize specialized knowledge, but conventional training requires abundant domain-specific corpora while high-quality algorithms are scarce in complex algorithm-design scenarios. In this paper, we propose sample-efficient parametric self-evolution where agents can explore and learn from self-generated algorithms. First, we characterize in-context evolutionary stagnation and analytically propose the Improvement Chain proposition, showing how learning successive self-generated algorithms can locally increase the likelihood of neighboring algorithms. Motivated by this local-transfer perspective, we further propose Population-Curated Policy Optimization (PCPO) to utilize a global population and a hybrid policy update scheme for retaining and reusing high-quality, diverse self-generated algorithms, shifting the policy towards stronger algorithms. In the task of learning rate schedule design for global placement in electronic design automation, trained only on 4 chip cases, PCPO outperforms the state-of-the-art in-context evolutionary methods (e.g., OpenEvolve and ShinkaEvolve) on average across 16 chip cases. With an 8B-size base model, PCPO achieves competitive performance compared to frontier closed-source models such as GPT-5.5. PCPO also reduces inference-time token cost by internalizing grounded domain knowledge and prompt distillation. Moreover, PCPO achieves significant speedups on four GPU kernel designs, with an average of 8.27$\times$ speedup against the PyTorch Eager baseline. 
\end{abstract}

\keywords{ Self-evolving Agent \and Evolutionary Learning \and Reinforcement Learning \and Automated Algorithm Design}

\section{Introduction}
As the capabilities of large language models (LLMs) continue to grow, LLMs are increasingly participating in complex real-world tasks in the form of agents~\citep{wang2024agentsurvey,xi2025agentsurvey}. The self-evolving algorithm-design agent~\citep{whitepaper} is one promising form for completing complex real-world tasks, where the agent solves them from a high-level decision-making perspective by iteratively designing and refining algorithms, achieving outstanding results in scenarios like classical combinatorial optimization problems and open mathematical discoveries with pure in-context evolutionary frameworks~\citep{EoH,alphaevolve,lange2026shinkaevolve}. 

\begin{figure*}[t]
\centering
\centering
\includegraphics[width=0.85\textwidth]{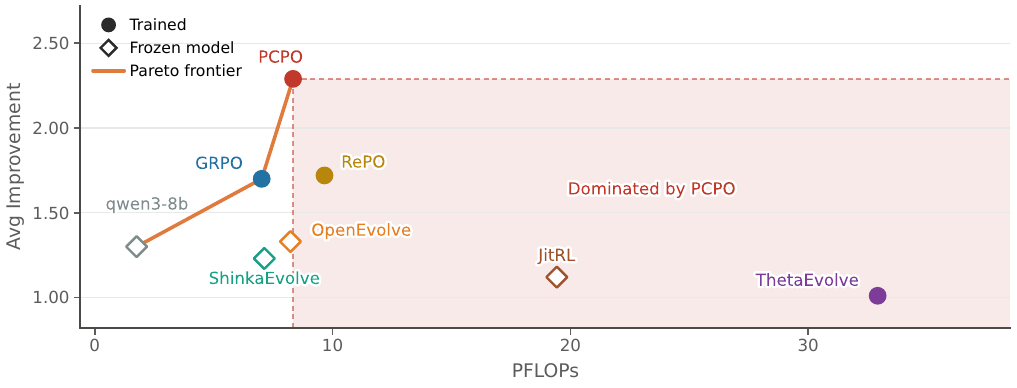} 
\caption{Average online improvement on 4 chips against total training and inference cost. PCPO achieves the largest improvement within acceptable computational cost, and the trained model can generalize to other chip cases with minimal inference-time token cost as shown in Section~\ref{exp-gp}. Frozen-model methods underperform the base model due to the stagnation described in Section~\ref{sec:in-context stagnation} and incur substantial computational cost due to repeated generation refinement and long thinking.}
\label{fig:pflops_vs_improvement}
\vspace{-12pt}
\end{figure*}

However, in complex scenarios that require specialized knowledge, pure in-context evolution tends to quickly plateau~\citep{evo-plateaus,lu2026advancing,lange2026shinkaevolve}, whose diagnostic experiments are shown in Figure~\ref{fig:diagnoses-qwen3}, posing obstacles to constructing effective self-evolving algorithm-design agents. A natural way to cope with specialized knowledge is to embed abundant domain-specific corpora in the model weights with continual pre-training (CPT) or supervised fine-tuning (SFT)~\citep{Domain-Knowledge-Injection,guo-etal-2025-efficient-domain, Domain-adaptive-Post-training}. However, algorithm-design scenarios such as global placement~\citep{chu2009placement} in electronic design automation (EDA) and GPU kernel optimization~\citep{kernelbench} offer insufficient labeled, high-quality algorithms. Therefore, self-evolving algorithm-design agents that automatically explore, enhance, and parametrically learn from self-generated algorithms demonstrate great significance. Recently, some works have made attempts to enhance evolution with reinforcement learning (RL) at inference time~\citep{wang2026thetaevolve, ttt-discover}, but they still face the typical sample-efficiency challenge of RL and do not make full use of the guidance from high-quality and diverse self-generated algorithms, leaving the potential benefits of population-level experience reuse across policy updates underexplored.

Therefore, in this work, we develop self-evolving algorithm-design agents that evolve via exploring and training on self-generated algorithms. We first characterize in-context evolutionary stagnation through diagnostic experiments. The observed sensitivity to domain information and the repeated low-quality generations motivate the hypothesis that hallucinated self-generated context can reinforce a cycle of biased outputs and local stagnation. Then we introduce the Improvement Chain as a motivating abstraction for studying local probability transfer between neighboring algorithms, with a conditional proposition that provides a conceptual account of how learning from curated intermediate algorithms may locally shift the task-conditional distribution toward neighboring algorithms. Motivated by this local-transfer perspective, we further propose Population-Curated Policy Optimization (PCPO) to utilize a global population for retaining and reusing high-quality, diverse self-generated algorithms as promising candidates. PCPO also adopts a hybrid policy update scheme, where refreshing on-policy updates enhance exploration and fast correction, while regularized off-policy updates learn elite self-generated algorithms, shifting the policy towards stronger neighbors.

We evaluate PCPO on the task of global placement (a critical procedure of the chip physical design process, deciding an overall layout of cells on the chip canvas~\citep{markov2012progress}), designing learning rate schedules for a GPU-accelerated placer, DREAMPlace~\citep{lin2020dreamplace}, which have been empirically shown to have a great impact on the final performance and convergence~\citep{agnesina2023autodmp}. Figure~\ref{fig:pflops_vs_improvement} shows the online performance comparisons, and PCPO achieves the largest improvement within acceptable computational cost in this online setting. Moreover, trained on only 4 chip cases, PCPO generalizes to other 12 chip cases, outperforming the state-of-the-art in-context evolutionary methods (e.g., OpenEvolve and ShinkaEvolve~\citep{lange2026shinkaevolve}) on average across 16 chip cases, and shows better sample efficiency than Group Relative Policy Optimization (GRPO)~\citep{deepseek-r1}. Based on an 8B-size model, PCPO achieves competitive performance compared to frontier closed-source models such as GPT-5.5, beating GPT-5.5 on training cases while slightly underperforming GPT-5.5 on testing cases. PCPO also reduces inference-time token cost by internalizing grounded domain knowledge and prompt distillation. After training, the policy can be decoded without thinking and self-generated references, thus saving tokens on a long chain of misleading rationales and prompts. We also evaluate PCPO on an extended application, GPU kernel design, in an online setting. PCPO achieves significant speedups on 4 cases of KernelBench~\citep{kernelbench}, with an average 8.27$\times$ speedup against the PyTorch Eager baseline.

Our contributions can be summarized as three aspects:

(1) We empirically characterize in-context evolutionary stagnation in two specialized algorithm-design domains and present diagnostic evidence consistent with the hypothesis that repeatedly conditioning on flawed self-generated context can reinforce unproductive generation patterns, reducing the probability of producing high-quality algorithms.

(2) We introduce the Improvement Chain as a motivating abstraction with a conditional local-transfer effect: under log-probability closeness and likelihood-gain assumptions, learning a visited algorithm increases a lower-bound change in the likelihood of a neighboring algorithm. Motivated by the conditional proposition, we propose PCPO, which has a curated global population to retain and reuse high-quality, diverse self-generated algorithms, refreshing on-policy updates to enhance exploration, and periodic regularized off-policy updates to shift the policy towards stronger neighboring algorithms.

(3) We evaluate PCPO on the task of learning rate schedule design for global placement and GPU kernel design. PCPO improves empirical performance and inference-time token efficiency over the state-of-the-art in-context LLM-driven evolutionary methods and is even competitive with frontier closed-source models, demonstrating the benefit of population-curated parametric updates.

\section{Background}
LLM-driven evolutionary search has become a standard recipe for algorithm design in open problems such as scientific discovery. Basically, in-context evolutionary methods have three key components: population, selection strategy, and operator, which function as external context management. 

\textbf{Population }The population can be regarded as a delicate memory management mechanism that balances the quality and diversity of the produced algorithms: 
$\mathcal{A}_{\text{pop}}:=\{(\boldsymbol{A}_i,\boldsymbol{I}_i)\}_{i=1}^M$, where $\boldsymbol{A}_i$ and $\boldsymbol{I}_i$ denote the $i$-th algorithm in the population and its corresponding information that includes execution feedback, and $M$ is the population size. 

\textbf{Selection strategy }The selection strategy can be regarded as a delicate selective context injection mechanism that selects the most important references and information from the memory needed by the agents at the current stage: 
$\{\boldsymbol{A}_{\text{ref}}^i,\boldsymbol{I}_{\text{ref}}^i\}_{i=1}^k=\text{select}(\mathcal{A}_{\text{pop}})$, where $k$ denotes the number of selected algorithms. 

\textbf{Operator }The operators $C_\text{op}$ can be regarded as guiding instructions (e.g., explore new programs). 

At iteration $t$, in-context evolutionary methods select reference algorithms $\{\boldsymbol{A}_{\text{ref}}^i,\boldsymbol{I}_{\text{ref}}^i\}_{i=1}^k$ and construct $C_\text{op}$ to prompt the LLM to produce new algorithm designs that are later evaluated and written back into the population: $\boldsymbol{A}_\text{new}=\text{LLMs}(C_\text{task}, C_\text{op},\{\boldsymbol{A}^i_\text{ref}, \boldsymbol{I}^i_\text{ref}\}^k_{i=1})$, where $C_\text{task}$ is the context of the task prompt and $(C_\text{op},\{\boldsymbol{A}^i_\text{ref}, \boldsymbol{I}^i_\text{ref}\}^k_{i=1})$ can be regarded as the additional context $C^t_{\mathrm{extra}}$ at step $t$.

\section{Related Work} 
\textbf{Evolutionary Algorithm Design:} AlphaEvolve~\citep{alphaevolve} casts algorithm discovery as evolutionary search over an archive of programs, where a frozen LLM mutates selected parents and an island-based MAP-Elites archive balances quality against diversity. ShinkaEvolve~\citep{lange2026shinkaevolve} retains this loop and redesigns its parent selection, novelty rejection, and LLM routing. The evolutionary framework was first introduced by FunSearch~\citep{FunSearch}. EoH~\citep{EoH} and ReEvo~\citep{ReEvo} extend it to heuristic code. ThetaEvolve~\citep{wang2026thetaevolve} and TTT-Discover~\citep{ttt-discover} further internalize the search with test-time RL.

\textbf{Off-Policy GRPO:} On-policy GRPO discards each group after one update, which is costly when generation dominates the budget. Since a clipped off-policy GRPO surrogate still lower-bounds reward improvement when the behavior policy stays nearby~\citep{IBM-offpolicy}, RePO~\citep{li2025repo} adds an off-policy term to recover signal when on-policy advantages collapse. However, how the replay buffer is managed is first-order~\citep{Meta-offpolicy}: reuse is a trade-off among staleness and diversity, and a well-designed buffer can match on-policy accuracy at much lower cost.

\textbf{Self-Evolving Agent:} A self-evolving agent persistently rewrites its parameters, context, tools, or architecture from its own trajectories and feedback, aiming to improve future performance~\citep{self-evolving-survey-1, self-evolving-survey-2}.
Current systems do so mainly in three ways: evolving the harness (memory, tools, or workflows)~\citep{zhang2026memevolve}; updating the model with SFT or RL on self-generated experience~\citep{zhang2026selfevolving-offline-rl,wang2026thetaevolve,ttt-discover}; or selecting and mutating a population of programs at test-time~\citep{chen2026maxproof}.

\section{From In-Context Stagnation to Parametric Self-Evolution}
\subsection{An Intuitive Analysis of In-Context Evolutionary Stagnation}\label{sec:in-context stagnation}
Algorithms produced by LLMs in the form of canonicalized program representation can be modeled as $A_{\mathrm{output}}=\arg\max_A P_{\mathrm{LLM}}(A\mid C_{\mathrm{task}})$ given the task context $C_{\mathrm{task}}$, while in-context evolutionary methods inject operators, selected algorithms, and feedback as additional context $C^t_{\mathrm{extra}}$, yielding $A_t=\arg\max_A P_{\mathrm{LLM}}(A\mid C_{\mathrm{task}}, C^t_{\mathrm{extra}})$ at step $t$. Although this conditioning enables iterative search, it can sometimes stagnate when specialized domain knowledge is required to generate the optimal algorithm $A^*$. In particular,
{
\setlength{\abovedisplayskip}{4pt}
\setlength{\belowdisplayskip}{4pt}
\small
\begin{equation}
    \frac{P_{\mathrm{LLM}}(A^*\mid C_{\mathrm{task}},C^t_{\mathrm{extra}})}{P_{\mathrm{LLM}}(A^*\mid C_{\mathrm{task}})}
=
\frac{P_{\mathrm{LLM}}(C^t_{\mathrm{extra}}\mid C_{\mathrm{task}}, A^*)}
{P_{\mathrm{LLM}}(C^t_{\mathrm{extra}}\mid C_{\mathrm{task}})}.
\end{equation}
}Because $C^t_{\mathrm{extra}}$ largely consists of the frozen model's previous outputs, misleading reasoning may be repeatedly reintroduced. Such context remains probable under the same model, but its probability may drop sharply conditioned on $A^*$ if it is negatively associated with $A^*$, i.e., hallucinated self-generated content inside $C^t_{\mathrm{extra}}$ logically conflicts with the implementation of $A^*$. Under that condition, the probability of generating $A^*$ using in-context evolutionary methods decreases compared to simple direct generations, creating a self-reinforcing cycle of biased outputs and local stagnation. Practical examples consistent with this hypothesis are provided in Appendix~\ref{appendix-stagnation}.

We also construct some diagnostic experiments, whose results are consistent with the hypothesis. In Figure~\ref{ood}, we compare the best results within 16 rollouts of the full prompt with the prompt that excludes detailed domain knowledge description, and ``$\times$'' denotes failure to generate valid algorithms that satisfy the task constraints in all 16 rollouts. The performance drop relative to the full prompt shows that the base model relies on prompt-provided domain knowledge to improve performance rather than intrinsic knowledge on both tasks. Figure~\ref{stagnation} illustrates the in-context evolutionary stagnation, where in-context evolutionary methods underperform the base model and fail to make improvements along the evolution process. One possible explanation for this performance drop is that self-generated context reinforces unproductive generation patterns. To show the generality of these phenomena, cross-model experiments are provided in Appendix~\ref{appendix-cross-model}. 

\begin{figure}[t]
  \centering
  \begin{subfigure}[b]{0.45\textwidth}
    \centering
    \includegraphics[width=0.8\textwidth]{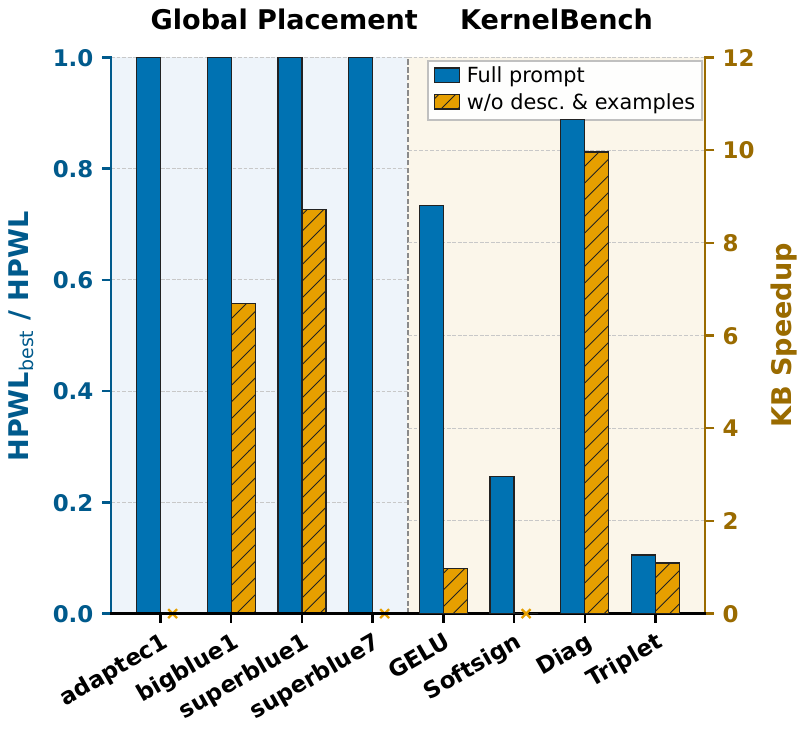}
    \setlength{\abovecaptionskip}{2pt}
    \caption{An illustration of relying on external domain knowledge rather than intrinsic knowledge.}
    \label{ood}
  \end{subfigure}
  \hfill
  \begin{subfigure}[b]{0.45\textwidth}
    \centering
    \includegraphics[width=0.8\textwidth]{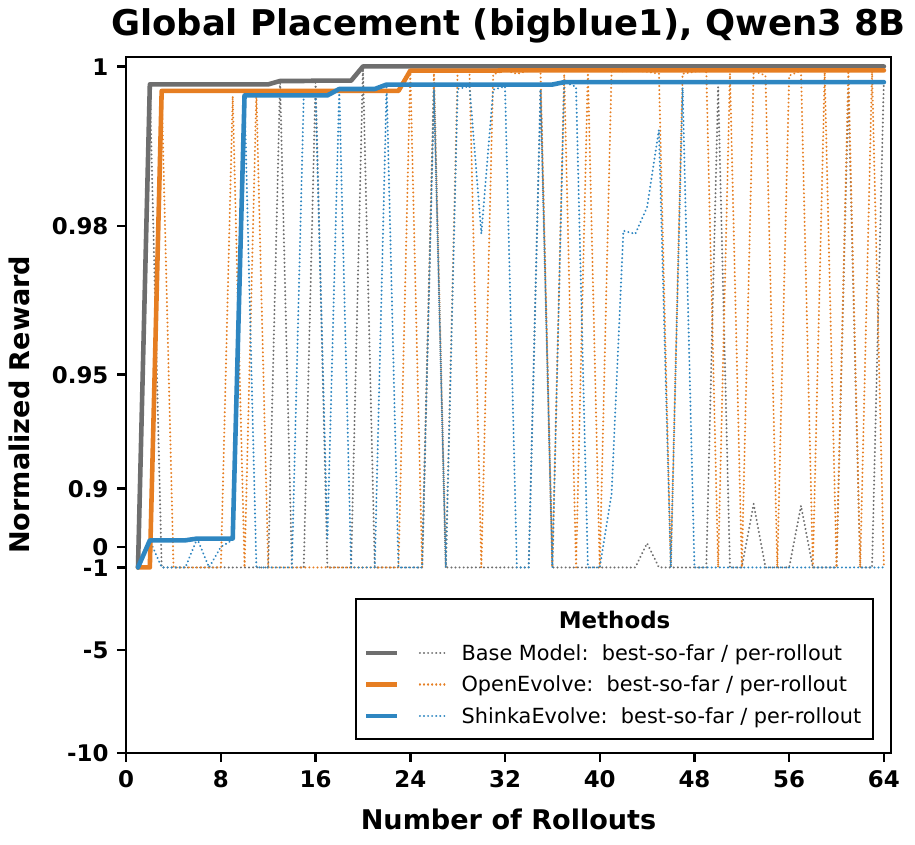}
    \setlength{\abovecaptionskip}{2pt}
    \caption{An illustration of the in-context evolutionary stagnation of OpenEvolve and ShinkaEvolve.}
    \label{stagnation}
  \end{subfigure}
  \setlength{\abovecaptionskip}{4pt}
  \caption{Diagnostic experiments based on Qwen3-8B.}
  \label{fig:diagnoses-qwen3}
  \vspace{-8pt}
\end{figure}

\subsection{Improvement Chain as a motivating abstraction}\label{IC}
As discussed in Section~\ref{sec:in-context stagnation}, repeatedly conditioning on misleading self-generated context may reduce the probability of producing high-quality algorithms in tasks that require domain knowledge. We therefore introduce the Improvement Chain as a conceptual model of parametric self-evolution, where, if each successor in the chain gradually approaches $A^*$, $P_{\mathrm{LLM}}(\cdot \mid C_{\mathrm{task}})$ is shifted towards more promising algorithms via parametric updates on curated self-generated algorithms, as shown in Figure~\ref{fig:IC-illustration}. Then, the central question we face is that, if training increases the likelihood of a visited algorithm, under what conditions does this update also increase the likelihood of a nearby algorithm?

\begin{figure}[t]
    \centering
    \includegraphics[width=0.8\textwidth]{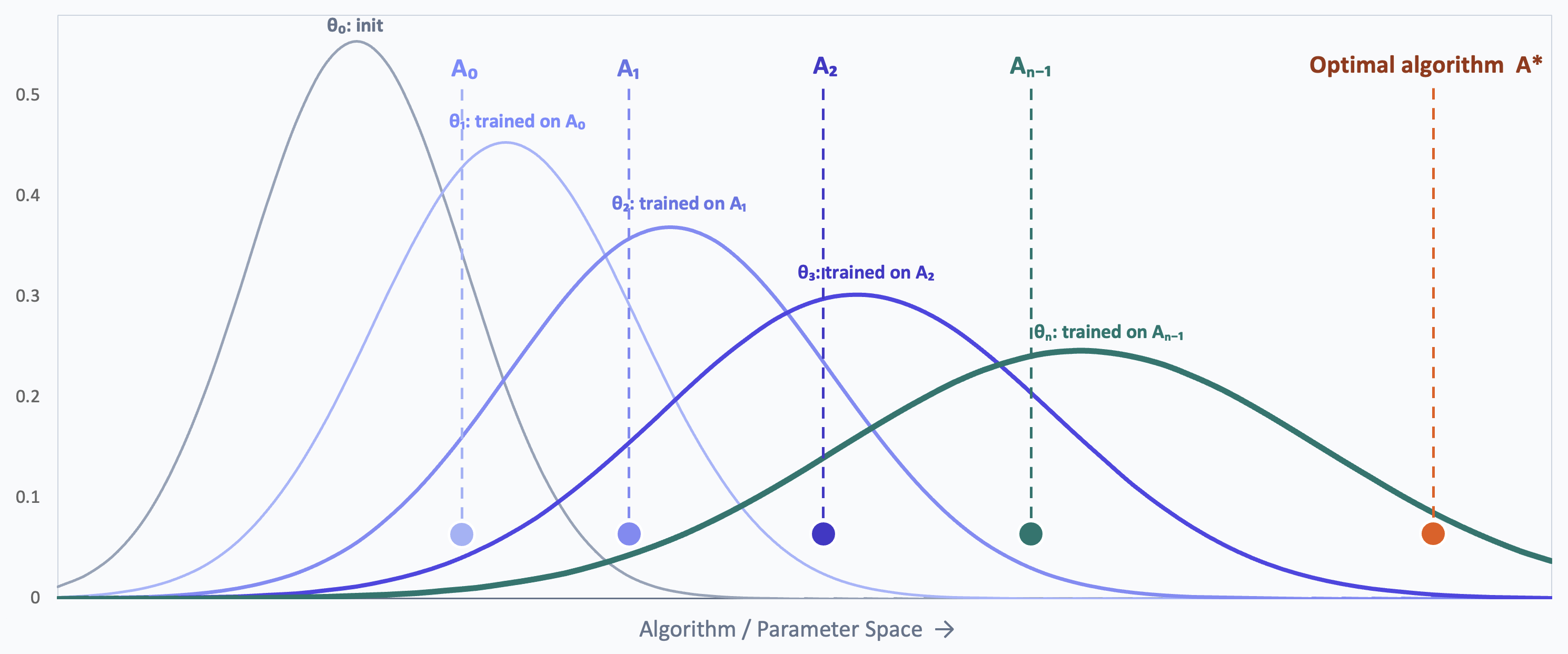}
    \caption{An illustration of the ideal Improvement Chain where each successor in the chain gradually approaches $A^*$. The optimal algorithm is approached by progressively increasing the reachability of successive chain nodes, generating and learning on chain nodes one after another.}
    \label{fig:IC-illustration}
    \vspace{-8pt}
\end{figure}

\textbf{Setup:} Let $\langle A_0, A_1, \ldots, A_n \rangle$ be an improvement chain with $A_n = A^*$. Note that in an ideal chain, $A_{i+1}$ does not necessarily have a better fitness score than $A_i$, but is more promising with respect to approaching the optimal $A^*$. This ordering represents a hypothesized target-directed relationship external to the proposition.
For each chain step $i \in \{0,1,\ldots,n-1\}$, let $\theta_i$ denote the model parameters
before learning on $A_i$, and let the update operator (i.e., training $\theta_i$ with $A_i$) produce
\begin{equation}
    \theta_{i+1}=\mathcal{U}(\theta_i;A_i).
    \label{eq:ic-update}
\end{equation}
For a fixed prompt context $C$, define
\begin{equation}
    \ell_\theta(A\mid C)\triangleq \log p_\theta(A\mid C).
    \label{eq:ic-logprob}
\end{equation}
We assume that the chain nodes have positive probability under the policies considered, so
that their log-probabilities are finite. Here, $A$ denotes a fixed, canonicalized program representation.

\textbf{From edit distance to probability closeness:} We first formalize the algorithmic neighborhood in sequence space.
Let $d(\cdot,\cdot)$ denote a token-level discrepancy metric  over algorithms. We require that chain edges are short edits, which is common in real iterative refinement conducted by agents:
\begin{equation}
    d(A_i, A_{i+1}) \le \Delta,
    \qquad \forall i \in \{0,1,\ldots,n-1\}.
    \label{eq:ic-edit-distance}
\end{equation}
The probability of generating adjacent algorithms with no significant logical differences from the same model usually remains similar, motivating Assumption~\ref{assump:ic-edit-lipschitz} for adjacent nodes in the chain.

\begin{assumption}[Local log-probability closeness before and after the update]
\label{assump:ic-edit-lipschitz}
There exists $\kappa>0$ such that, for every \(i\in\{0,\ldots,n-1\}\) and \(\theta\in\{\theta_i,\theta_{i+1}\}\),
\begin{equation}
    \left|
        \ell_{\theta}(A_i\mid C)
        -\ell_{\theta}(A_{i+1}\mid C)
    \right|
    \leq \kappa \cdot d(A_i,A_{i+1}).
    \label{eq:ic-local-smoothness}
\end{equation}
\end{assumption}

Since training the model on a visited algorithm increases the log-likelihood of generating it, we define $\epsilon_\ell$ as the smallest of such gains along the chain. That is, for every $i\in\{0,\ldots,n-1\}$,
\begin{equation}
    \ell_{\theta_{i+1}}(A_i\mid C)
    -\ell_{\theta_i}(A_i\mid C)
    \geq \epsilon_\ell.
    \label{eq:ic-self-gain}
\end{equation}


Then, utilizing Eqs.~\ref{eq:ic-edit-distance} to~\ref{eq:ic-self-gain}, i.e., that the gap between the log-probabilities of $A_i$ and $A_{i+1}$ is upper bounded by $\kappa \Delta$ and an update raises the log-likelihood of $A_i$ by at least $\epsilon_\ell$,  we can derive Proposition~\ref{lem:ic-spillover}, showing that the log-likelihood increment of generating $A_{i+1}$ after an update is lower bounded by $\epsilon_\ell-2\kappa \Delta$. Proofs are provided in Appendix~\ref{appendix-proof}.

\begin{proposition}[Local likelihood transfer]
\label{lem:ic-spillover}
Under Assumptions~\ref{assump:ic-edit-lipschitz}, for every
$i\in\{0,\ldots,n-1\}$,
\begin{equation}
    \ell_{\theta_{i+1}}(A_{i+1}\mid C)
    -\ell_{\theta_i}(A_{i+1}\mid C)
    \geq
    \epsilon_\ell-2\kappa d(A_i,A_{i+1})
    \geq
    \epsilon_\ell-2\kappa\Delta.
    \label{eq:ic-log-transfer}
\end{equation}
Consequently, if $\rho_\ell\triangleq\epsilon_\ell-2\kappa\Delta>0$, then
\begin{equation}
    p_{\theta_{i+1}}(A_{i+1}\mid C)
    \geq
    e^{\rho_\ell}p_{\theta_i}(A_{i+1}\mid C).
    \label{eq:ic-multiplicative-transfer}
\end{equation}
\end{proposition}

Therefore, Proposition~\ref{lem:ic-spillover} gives a sufficient condition under which training on the current algorithm increases the likelihood of a nearby algorithm. In particular, if the direct gain $\epsilon_\ell$ on the log-likelihood of $A_i$ exceeds twice the worst-case gap $\kappa\Delta$ between the log-probabilities of $A_i$ and $A_{i+1}$, the successor receives a multiplicative probability improvement. If each successor in the chain gradually approaches $A^*$, such local probability transfer could increase the chance of sampling successors one after another, facilitating the traversal of the whole chain and shifting the task-conditional distribution of the base model towards the optimal algorithm. Appendix~\ref{appendix-gp-case} provides a practical example to show the actual magnitude of the values of the terms $\epsilon_\ell$ and $\kappa\Delta$, satisfying $\rho_\ell>0$.

\section{Practical Algorithm: Population-Curated Policy Optimization}

\begin{figure}[t]
\centering
\includegraphics[width=0.9\textwidth]{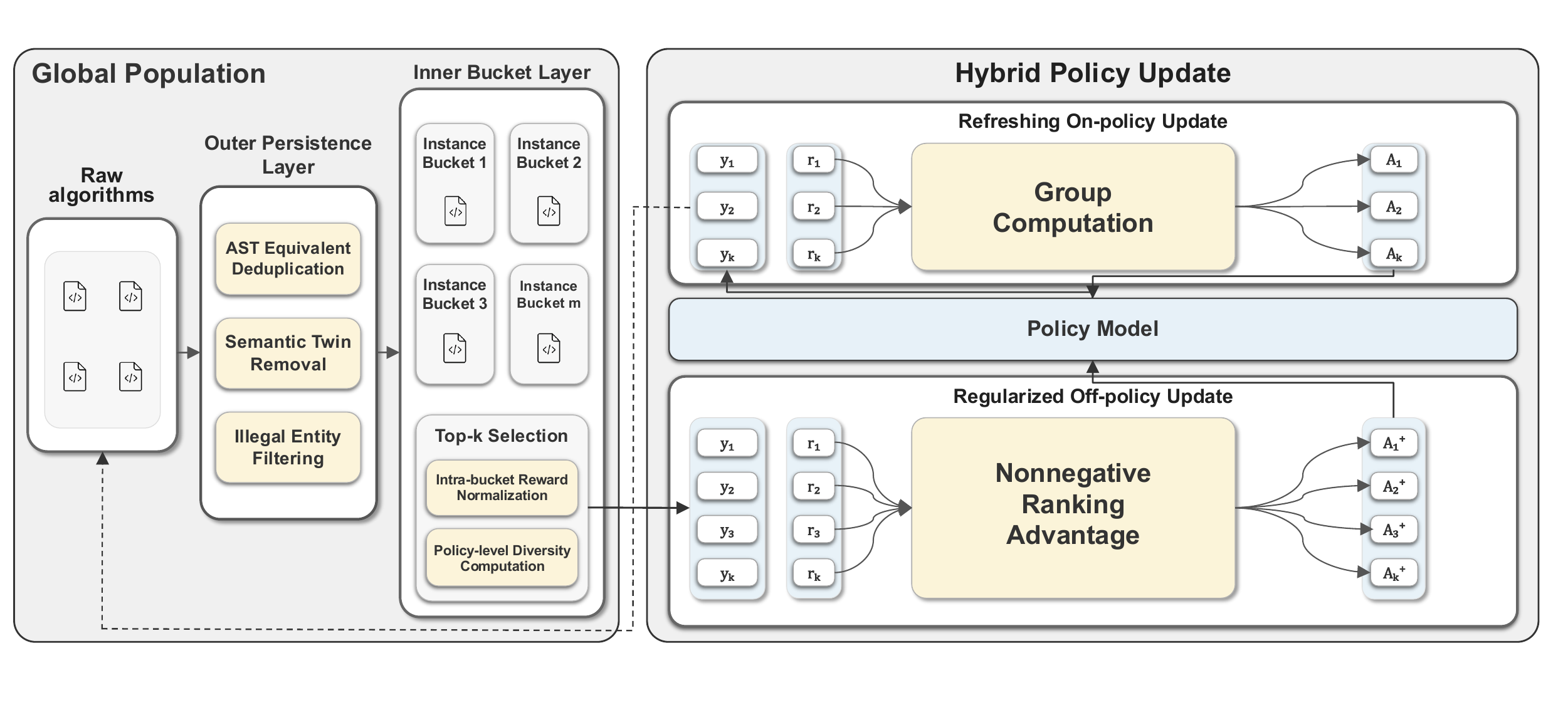} 
\setlength{\abovecaptionskip}{0pt}
\caption{The workflow of PCPO. The left part illustrates the Global Population, filtering promising candidates while avoiding incompatible reward scales and difficulties. The right part illustrates the Hybrid Policy Update scheme, supporting sustained improvement, exploration, and rapid correction.}
\label{fig:workflow}
\vspace{-8pt}
\end{figure}

\paragraph{From Ideal Chains to PCPO}
Section~\ref{IC} provides a stylized account of one possible local-transfer effect but does not establish monotonic fitness improvement, guarantee the existence or discovery of an improvement chain, or provide a convergence guarantee. Instead, it motivates a practical design principle: retain useful algorithms, update the policy with promising elites, explore diverse variants in their neighborhoods, and shift the policy towards stronger algorithms. We therefore propose Population-Curated Policy Optimization as follows: retain a verified program \(A_i\) as a high-quality candidate in the global population, encourage
exploration beyond near-duplicate programs by structural-diversity curation, approximate the likelihood-increasing update by positive off-policy reinforcement, and search for potential successors by fresh on-policy sampling.


\subsection{Global Population} 
As shown in the left part of Figure~\ref{fig:workflow}, Global Population manages evaluated population experiences by retaining and reusing high-quality, diverse self-generated algorithms for focused learning. It consists of an Outer Persistence Layer and an Inner Bucket Layer.

\textbf{The Outer Persistence Layer} maintains a compact, quality-controlled archive $\mathcal{P}$, allowing valuable feedback to be reused beyond a single on-policy rollout. It removes redundant or invalid candidates through AST Equivalent Deduplication, Semantic Twin Removal, and Illegal Entity Filtering.

AST Equivalent Deduplication removes logically identical programs that differ only in presentation, such as comments, docstrings, or whitespace. For each code $c$, we parse its abstract syntax tree (AST), remove leading module/class/function docstrings, serialize the normalized tree, and hash it to obtain an equivalence key $\kappa(c)$. Codes sharing the same key are treated as duplicates, preventing superficial variants from consuming archive capacity or dominating Top-$k$ selection.

Semantic Twin Removal prevents the archive from collapsing into near-identical elites after exact deduplication. We compute a structural distance $\mathrm{Div}_{\mathrm{struct}}(c,c')$ from ASTs normalized by docstring removal and $\alpha$-renaming of bound identifiers. Given a keeper $h$ (i.e., an elite already in the population) and a candidate $g$ satisfying $\mathrm{Div}_{\mathrm{struct}}(g,h)\leq\delta_{\max}$, $g$ is retained only if $\mathrm{fitness}(g)\geq\mathrm{fitness}(h)+\gamma_{\min}$. Thus, small structural variations are preserved only when they provide a meaningful fitness improvement, maintaining both archive quality and diversity.

Illegal Entity Filtering serves as a hard validity gate: programs with syntax or execution failures, or those violating task-specific constraints, are excluded from $\mathcal{P}$ and subsequent Top-$k$ curation. 

\textbf{The Inner Bucket Layer} partitions the global archive into instance-specific buckets $\{\mathcal{P}_b\}_{b=1}^{m}$, preventing incompatible reward scales and task difficulties from being mixed directly while retaining the sample-efficiency benefits of multi-instance training. Within each bucket, candidates are ranked for regularized off-policy learning using normalized reward and policy-level diversity.

Intra-bucket Reward Normalization rescales eligible fitness values to $[\varepsilon,1]$ while preserving their within-bucket ranking. For a candidate $c\in\mathcal{P}_b$ with fitness $r$, let $\mathcal{R}_b$ denote the eligible rewards within the bucket $b$ and define $r_{\min}=\min\mathcal{R}_b$ and $r_{\max}=\max\mathcal{R}_b$. Then we have
{
\setlength{\abovedisplayskip}{4pt}
\setlength{\belowdisplayskip}{4pt}
\small
\begin{equation}
    \tilde{r}_b(c)
    =
    \begin{cases}
    1, & \text{if } r_{\max}=r_{\min}, \\
    \varepsilon+(1-\varepsilon)\dfrac{r-r_{\min}}{r_{\max}-r_{\min}}, & \text{otherwise},
    \end{cases}
    \end{equation}
}where $\varepsilon>0$ ensures a strictly positive quality score.

Policy-level Diversity Computation measures how structurally different an archived candidate is from codes recently generated by the current policy on the same instance (e.g., the same chip case): 
{
\setlength{\abovedisplayskip}{4pt}
\setlength{\belowdisplayskip}{4pt}
\small
\begin{equation}
        d_b(c)
        =
        \begin{cases}
        \dfrac{1}{|\mathcal{C}_b^{(t)}|}
        \sum_{c'\in\mathcal{C}_b^{(t)}}
        \mathrm{Div}_{\mathrm{struct}}(c,c'),
        & \text{if }\mathcal{C}_b^{(t)}\neq\emptyset, \\
        \mathrm{Div}_{\mathrm{struct}}(c,\,c_b^*),
        & \text{otherwise}.
        \end{cases}
    \end{equation}
}where $\mathcal{C}_b^{(t)}$ is the set of distinct eligible codes generated by the current policy at step $t$ for the instance bucket $b$, and $c^*_b$ is the code with the best reward inside the bucket $b$. A larger $d_b(c)\in[0,1]$ indicates greater novelty relative to the current policy.

Finally, candidates are ranked by the weighted score $s_b(c)=\alpha\cdot\,\tilde{r}_b(c)+\beta_\text{div}\cdot\,d_b(c)$, and the Top-$k$ candidates from each bucket are selected as elites for the regularized off-policy update.

\subsection{Hybrid Policy Update}
As shown in the right part of Figure~\ref{fig:workflow}, Hybrid Policy Update learns promising candidates curated by Global Population and shifts the policy towards stronger neighbors while exploring new ones. It combines a Refreshing On-policy Update, performed at every step for exploration, with a Regularized Off-policy Update, performed every $T$ steps to reuse high-quality experiences.

\textbf{Regularized Off-policy Update}: Rather than using SFT~\citep{sft,sft-survey} to imitate the token-level form of elite codes, we treat selected elites as off-policy trajectories and reinforce them through sequence-level advantages. This encourages the policy to learn transferable improvement patterns~\citep{sft-rl, sun2026rl-code}. Because stale off-policy samples can cause distribution drift and ratio explosion~\citep{Meta-offpolicy, IBM-offpolicy}, we retain clipped importance ratios and KL regularization toward a fixed reference policy. Selecting recent and diverse elites as candidates further limits sample staleness. 

For the $k$ elites in bucket $b$, sorted by decreasing environment return, we introduce a Nonnegative Ranking Advantage. Assigning rank weight $w_i=k-i+1$ and setting $W=k(k+1)/2$, the advantage of elite $i$ is $\hat{A}_i^{+}
=\hat{A}_0\cdot\frac{w_i}{W}\cdot k
=\hat{A}_0\cdot\frac{2(k-i+1)}{k+1},\text{ }i=1,\text{ }\ldots,\text{ }k,\text{ }\hat{A}_0=1$. 
This positive rank profile has mean $\hat{A}_0$: stronger elites receive greater reinforcement, while every verified Top-$k$ candidate retains a positive advantage, avoiding the negative feedback that standard within-group GRPO may assign to valid but lower-ranked elites. The off-policy objective is
{
\setlength{\abovedisplayskip}{4pt}
\setlength{\belowdisplayskip}{4pt}
\small
\begin{equation}
\mathcal{J}_{\mathrm{off}}(\theta)
=
\frac{1}{k}
\sum_{i=1}^{k}
\frac{1}{|c_i|}
\sum_{t=1}^{|c_i|}
\Big[\min\!\Big(
    r_{i,t}(\theta) A_i^{+},
    \operatorname{clip}\!\bigl(
        r_{i,t}(\theta),
        1-\varepsilon_\text{clip},
        1+\varepsilon_\text{clip}
    \bigr) A_i^{+}
\Big)-\beta D_{\mathrm{KL}}(\pi_\theta\|\pi_{\mathrm{ref}})
\Big].
\label{eq:token-level-off-policy-objective}
\end{equation}
}
where
$r_{i,t}(\theta)
=
\frac{
    \pi_{\theta}\!\left(c_{i,t}\mid q,c_{i,<t}\right)
}{
    \pi_{\mathrm{old},i}\!\left(c_{i,t}\mid q,c_{i,<t}\right)
}$, $\pi_{\mathrm{old},i}$ generated elite $c_i$, and $\pi_{\mathrm{ref}}$ is the fixed reference policy.

\textbf{Refreshing On-policy Update}: At every step, Hybrid Policy Update applies standard GRPO with fresh rollouts and a smaller learning rate. Immediate environmental feedback promotes exploration and supplies new candidates to Global Population, while strong penalties rapidly suppress invalid or constraint-violating code. These fresh updates also anchor the policy to its current rollout distribution, mitigating the drift introduced by periodic off-policy updates.

\section{Experiment}
We evaluate our method on two algorithm-design scenarios, i.e., global placement in EDA~\citep{chu2009placement} and the extended experiment, GPU kernel design~\citep{kernelbench}. We adopt Qwen3-8B~\citep{yang2025qwen3} as the base model and equip it with in-context evolutionary harnesses like OpenEvolve (the open-source implementation of AlphaEvolve~\citep{alphaevolve}) and ShinkaEvolve~\citep{lange2026shinkaevolve}. We also provide GRPO~\citep{deepseek-r1} as a comparable baseline enhanced by reinforcement learning. Moreover, we further compare our method to frontier closed-source models such as GPT-5.5, as well as a general-purpose coding agent harness, Codex. Note that since ThetaEvolve~\citep{wang2026thetaevolve} is a test-time learning method, we provide global placement results in an online setting in Appendix~\ref{appedic-inference-time-exp}. Cross-model results are shown in Appendix~\ref{appendix-cross-model}.

\subsection{Global Placement} \label{exp-gp}
Global placement (GP) is a crucial step in chip placement~\citep{chu2009placement}, which has a direct and significant impact on the performance, power consumption, and area (PPA) of the final chip design, as it determines the locations of the cells within the given chip layout area~\citep{markov2012progress}. The goal of GP is to minimize half-perimeter wirelength (HPWL) while meeting the target overflow. In this experiment, we apply algorithm-design agents to design the learning rate schedules for the Adam optimizer~\citep{kingma2014adam} of the GPU-accelerated placer DREAMPlace~\citep{lin2020dreamplace}, which, as shown empirically, has a great impact on the quality and convergence of the GP process and is hard to tune~\citep{agnesina2023autodmp}. Designing the learning rate schedules requires domain knowledge, balancing the wire-length and cell density via changing the positions of cells through gradient backpropagation~\citep{lin2020dreamplace}. More detailed descriptions of GP are in Appendix~\ref{appendix-gp-intro}.

\begin{table*}[t]
  \centering
  \caption{Best@$n$ HPWL ($\times 10^6$) results on two global placement benchmarks. The lower, the better. The 4 cases used for RL training are underlined.}
  \vspace{-4pt}
  \label{tab:globalplacement-benchmark}
  \small
  \setlength{\tabcolsep}{3pt}
  \begin{adjustbox}{max width=\textwidth}
  \begin{tabular}{lccccccccccccccc}
    \toprule
     & \multicolumn{3}{c}{Qwen3-8B} & \multicolumn{3}{c}{OpenEvolve} & \multicolumn{3}{c}{ShinkaEvolve} & \multicolumn{3}{c}{GRPO} & \multicolumn{3}{c}{PCPO} \\
    \cmidrule(lr){2-4} \cmidrule(lr){5-7} \cmidrule(lr){8-10} \cmidrule(lr){11-13} \cmidrule(lr){14-16}
     & best@1 & best@4 & best@16 & best@1 & best@4 & best@16 & best@1 & best@4 & best@16 & best@1 & best@4 & best@16 & best@1 & best@4 & best@16 \\
    \midrule
    \underline{adaptec1}    & 71.11  & 71.09  & 71.01  & $\times$     & $\times$     & 70.54  & $\times$     & 71.58  & 71.58  & $\times$     & 71.22  & 70.51  & \textbf{70.80} & \textbf{70.36} & \textbf{70.15} \\
    adaptec2    & $\times$     & $\times$     & $\times$     & $\times$     & 80.27  & 79.35  & 82.64  & 81.04  & 81.04  & 80.18  & 80.11  & 80.11  & \textbf{80.12} & \textbf{78.83} & \textbf{78.83} \\
    adaptec3    & $\times$     & $\times$     & 187.64 & 234.01 & 233.98 & 233.98 & \textbf{187.71} & 187.71 & 187.71 & $\times$     & 187.90 & 186.94 & 187.82 & \textbf{186.40} & \textbf{186.25} \\
    adaptec4    & 171.45 & 171.45 & 171.36 & 206.43 & 206.43 & 206.43 & 175.33 & 173.18 & 173.18 & 172.12 & 171.69 & 170.55 & \textbf{171.40} & \textbf{171.40} & \textbf{168.26} \\
    \underline{bigblue1}    & \textbf{88.45} & 88.22  & 88.20  & $\times$     & 88.90  & 88.90  & $\times$     & 208.83 & 89.21  & 88.75  & 88.64  & 87.67  & 88.59  & \textbf{88.06} & \textbf{87.65} \\
    bigblue2    & $\times$     & 133.03 & 132.94 & \textbf{133.04} & 133.00 & 132.97 & $\times$     & 133.00 & 132.41 & $\times$     & 133.07 & 132.45 & $\times$     & \textbf{132.93} & \textbf{132.28} \\
    bigblue3    & $\times$     & 326.46 & 326.39 & $\times$     & $\times$     & $\times$     & 345.57 & \textbf{303.47} & 303.47 & $\times$     & 326.42 & 323.87 & \textbf{328.14} & 326.58 & \textbf{303.11} \\
    bigblue4    & $\times$     & $\times$     & $\times$     & 757.67 & 757.61 & 757.53 & \textbf{732.25} & \textbf{732.25} & 732.25 & $\times$     & 756.62 & 756.62 & 755.41 & 748.17 & \textbf{727.46} \\
    \underline{superblue1}  & $\times$     & \textbf{388.12} & 387.80 & \textbf{389.19} & 389.19 & 389.19 & $\times$     & 468.06 & 468.06 & 536.79 & 388.94 & 388.12 & 389.67 & 388.95 & \textbf{385.68} \\
    superblue3  & $\times$     & 469.40 & 469.40 & \textbf{462.31} & 462.31 & 460.57 & $\times$     & $\times$     & $\times$     & 465.64 & 461.37 & 461.23 & 463.14 & \textbf{461.36} & \textbf{455.29} \\
    superblue4  & $\times$     & 291.76 & 291.76 & $\times$     & $\times$     & $\times$     & $\times$     & $\times$     & $\times$     & $\times$     & 296.83 & 290.60 & \textbf{291.58} & \textbf{291.58} & \textbf{289.79} \\
    superblue5  & $\times$     & 454.36 & 452.02 & $\times$     & 459.10 & 459.10 & $\times$     & $\times$     & $\times$     & \textbf{452.33} & 452.33 & 451.94 & 457.61 & \textbf{452.20} & \textbf{450.76} \\
    \underline{superblue7}  & 728.17 & 723.62 & 723.62 & 556.00 & 556.00 & 556.00 & $\times$     & $\times$     & $\times$     & 707.16 & 707.16 & 549.55 & \textbf{551.07} & \textbf{549.45} & \textbf{547.43} \\
    superblue10 & 862.16 & 862.08 & 862.02 & 861.26 & 861.26 & 861.25 & $\times$     & 863.92 & 863.92 & 861.51 & 861.51 & 859.15 & \textbf{860.22} & \textbf{860.22} & \textbf{856.73} \\
    superblue16 & 401.68 & 401.36 & 401.32 & 401.35 & 401.35 & 401.35 & $\times$     & $\times$     & $\times$     & \textbf{401.03} & 401.03 & 400.68 & 401.59 & \textbf{400.73} & \textbf{400.31} \\
    superblue18 & 233.75 & 233.35 & 233.11 & 223.88 & 223.88 & 223.84 & $\times$     & $\times$     & 224.45 & 232.18 & 223.78 & 223.14 & \textbf{223.47} & \textbf{223.47} & \textbf{222.82} \\
    \midrule
    Avg rank    & 3.66   & 3.38   & 3.69   & 2.91   & 3.62   & 3.78   & 3.72   & 3.94   & 4.16   & 3.06   & 2.69   & 2.38   & \textbf{1.66}   & \textbf{1.38}   & \textbf{1.00}   \\
    \bottomrule
  \end{tabular}
  \end{adjustbox}
  \vspace{-12pt}
\end{table*}

As shown in Table~\ref{tab:globalplacement-benchmark}, we test PCPO on two popular open-source benchmarks for GP, ISPD 2005~\citep{ispd2005} and ICCAD 2015~\citep{iccad2015}. The metric best@$n$ shows the best result achieved within $n$ rollouts. ``$\times$" denotes that no valid result satisfying the task constraints is achieved within $n$ rollouts. For RL methods, models are trained on 4 representative chip cases within their own series and are then directly used for inference on all cases. PCPO achieves the best average rank on all best@$n$ metrics, which demonstrates its improvement within a small rollout budget and its stability on one-shot generations. These results indicate that PCPO effectively retains and reuses valuable and diverse experiences for parametric learning, inducing generalizable capabilities. Note that in-context evolutionary methods show worse average performance compared to the base model, which is consistent with the hypothesis made in Section~\ref{sec:in-context stagnation}. Figure~\ref{fig:mix_training} further shows PCPO's sample efficiency in mixed-case training brought by population experience reuse compared to GRPO. PCPO also outperforms RePO~\citep{li2025repo}, a hybrid-update RL method, suggesting the importance of balancing the quality and diversity of population experiences, as shown in Table~\ref{tab:repo}. As shown in Figure~\ref{fig:token}, PCPO also reduces inference-time token cost by internalizing grounded domain knowledge (no need to perform long‑range reasoning) and distilling the detailed prompt (e.g., in-context exemplars that evolutionary methods repeatedly inject). Thus, the cost of training is worthwhile because after a single training run, PCPO can achieve outstanding performance on broader cases in the specific domain with minimal token consumption. More ablation results are in Appendix~\ref{appendix-gp-ablation}. 

Moreover, to demonstrate the effectiveness of self-evolving via PCPO, we compare PCPO to baseline results (DREAMPlace~\citep{lin2020dreamplace} with human-set parameters), as well as GPT-5.5, GPT-4o, and Codex. We use the average best@16 of 16 independent runs for a fair and stable comparison. As shown in Table~\ref{tab:gp-strong}, trained on only 4 cases, PCPO outperforms baseline results with an average improvement of 1.54\%, surpassing GPT-4o and Codex. Compared to GPT-5.5, PCPO also achieves comparable average results (i.e., beating GPT-5.5 on training cases while slightly underperforming it on testing cases), with better results on 10 cases based on the 8B base model. Note that Codex can be seen as another in-context evolutionary method because it automatically selects candidate codes and improves them. The performance degradation is also consistent with the hypothesis that hallucinated self-generated context can reinforce a cycle of biased outputs and local stagnation.
\begin{table*}[t]
  \centering
  \caption{Average best@16 HPWL ($\times 10^6$) of 16 independent runs and relative improvement. The 4 cases used for RL training are underlined.}
  \vspace{-4pt}
  \label{tab:gp-strong}
  \setlength{\tabcolsep}{3.5pt}
  \begin{adjustbox}{max width=0.8\textwidth}
  \begin{tabular}{l|c|cc|cc|cc|cc}
    \toprule
     \multirow{2}{*}{Chip case}
     & DREAMPlace
     & \multicolumn{2}{c|}{GPT-4o}
     & \multicolumn{2}{c|}{GPT-5.5}
     & \multicolumn{2}{c|}{Codex-GPT-5.5}
     & \multicolumn{2}{c}{PCPO-Qwen3-8B} \\
     \cmidrule(lr){3-4} \cmidrule(lr){5-6} \cmidrule(lr){7-8} \cmidrule(lr){9-10}
     & HPWL & best@16 & Imp. (\%) & best@16 & Imp. (\%) & best@16 & Imp. (\%) & best@16 & Imp. (\%) \\
    \midrule
    \underline{adaptec1} & 73.27 & 71.04 & +3.04\% & 70.32 & +4.03\% & 73.61 & -0.46\% & \textbf{70.18} & \textbf{+4.22\%} \\
    adaptec2 & 82.28 & 84.63 & -2.86\% & 80.16 & +2.58\% & 79.84 & +2.96\% & \textbf{78.92} & \textbf{+4.08\%} \\
    adaptec3 & 189.17 & 186.20 & +1.57\% & 186.19 & +1.57\% & 203.45 & -7.55\% & \textbf{186.14} & \textbf{+1.60\%} \\
    adaptec4 & 172.15 & \textbf{167.56} & \textbf{+2.67\%} & 169.12 & +1.76\% & 230.42 & -33.85\% & 168.34 & +2.21\% \\
    \underline{bigblue1} & 88.93 & \textbf{87.28} & \textbf{+1.86\%} & 87.54 & +1.56\% & 169.12 & -90.18\% & 87.43 & +1.69\% \\
    bigblue2 & 133.26 & \textbf{131.91} & \textbf{+1.01\%} & \textbf{131.91} & \textbf{+1.01\%} & 132.54 & +0.54\% & 132.24 & +0.77\% \\
    bigblue3 & 305.24 & 340.30 & -11.49\% & \textbf{292.92} & \textbf{+4.04\%} & 296.77 & +2.78\% & 302.98 & +0.74\% \\
    bigblue4 & 737.65 & 764.68 & -3.66\% & 736.09 & +0.21\% & 837.54 & -13.54\% & \textbf{727.29} & \textbf{+1.40\%} \\
    \underline{superblue1} & 390.08 & 389.08 & +0.26\% & 385.69 & +1.12\% & 390.71 & -0.16\% & \textbf{385.56} & \textbf{+1.16\%} \\
    superblue3 & 462.22 & 445.43 & +3.63\% & \textbf{444.38} & \textbf{+3.86\%} & 512.66 & -10.91\% & 455.34 & +1.49\% \\
    superblue4 & 292.11 & 289.07 & +1.04\% & 289.25 & +0.98\% & \textbf{288.97} & \textbf{+1.07\%} & 289.82 & +0.78\% \\
    superblue5 & 452.34 & 451.06 & +0.28\% & \textbf{449.88} & \textbf{+0.54\%} & 451.11 & +0.27\% & 450.57 & +0.39\% \\
    \underline{superblue7} & 556.69 & 547.54 & +1.64\% & 551.60 & +0.91\% & 549.66 & +1.26\% & \textbf{547.21} & \textbf{+1.70\%} \\
    superblue10 & 861.29 & 861.12 & +0.02\% & 856.78 & +0.52\% & 856.82 & +0.52\% & \textbf{856.52} & \textbf{+0.55\%} \\
    superblue16 & 402.58 & 400.54 & +0.51\% & \textbf{399.76} & \textbf{+0.70\%} & 399.97 & +0.65\% & 400.08 & +0.62\% \\
    superblue18 & 224.86 & 230.06 & -2.31\% & 222.82 & +0.91\% & \textbf{219.07} & \textbf{+2.58\%} & 221.94 & +1.30\% \\
    \midrule
    Avg-train & 1.000 & 0.983 & +1.70\% & 0.981 & +1.91\% & 1.224 & -22.38\% & \textbf{0.978} & \textbf{+2.19\%} \\
    Avg-test & 1.000 & 1.008 & -0.80\% & \textbf{0.984} & \textbf{+1.56\%} & 1.045 & -4.54\% & 0.987 & +1.33\% \\
    Avg & 1.000 & 1.002 & -0.17\% & \textbf{0.984} & \textbf{+1.64\%} & 1.090 & -9.00\% & 0.985 & +1.54\% \\
    \bottomrule
  \end{tabular}
  \end{adjustbox}
\end{table*}

\begin{table}[t]
  \centering
  \captionsetup{belowskip=2pt}
  \caption{Average best speedups of 4 independent runs on 4 cases of KernelBench. Higher is better.}
  \label{tab:kernelbench-benchmark}
  \setlength{\tabcolsep}{4pt}
  \begin{adjustbox}{max width=0.83\textwidth}
  \begin{tabular}{lcccc|c}
    \toprule
    Method &
    \makecell{MinGPTNewGelu} &
    \makecell{SoftSign} &
    \makecell{MatmulWithDiagonalMatrices} &
    \makecell{TripletMarginLoss} &
    Avg rank \\
    \midrule
    Qwen3-8B     & $8.801\times$ & $2.957\times$ & $10.664\times$ & $1.267\times$ & 6.00 \\
    OpenEvolve  & $8.741\times$ & $3.373\times$ & $10.478\times$ & $1.277\times$ & 5.25 \\
    ShinkaEvolve & $8.727\times$ & $3.373\times$ & $10.478\times$ & $1.094\times$ & 6.00 \\
    ThetaEvolve & $8.986\times$ & $3.306\times$ & $15.960\times$ & $4.524\times$ & 4.00 \\
    \midrule
    GPT-5.5      & $9.018\times$ & $2.939\times$ & $\textbf{16.621}\times$ & $4.537\times$ & 3.38 \\
    Codex        & $8.712\times$ & $3.350\times$ & $\textbf{16.621}\times$ & $4.419\times$ & 4.38 \\
    \midrule
    GRPO         & $9.002\times$ & $3.071\times$ & $10.169\times$ & $3.400\times$ & 5.50 \\
    PCPO         & $\textbf{9.033}\times$ & $\textbf{3.446}\times$ & $16.067\times$ & $\textbf{4.551}\times$ & \textbf{1.50} \\
    \bottomrule
  \end{tabular}
  \end{adjustbox}
  \vspace{-12pt}
\end{table}

\subsection{Extended Experiment: GPU Kernel Design}
GPU kernel design is a hardware-aware optimization process that maps high-level tensor computations onto the parallel execution and memory hierarchy of a target accelerator, greatly influencing the GPU computational efficiency and energy consumption~\citep{tillet2019triton, spector2025thunderkittens}. We apply algorithm-design agents to design GPU kernels for 4 Level-1 cases highlighted in the original KernelBench paper~\citep{kernelbench}, covering distinct computational structures: transcendental-heavy elementwise computation, bandwidth-dominated pointwise computation, structured broadcasting, and multi-input reduction. We use the best speedup within the rollout budget as the metric. For RL methods, we test their inference-time learning ability by granting them the same rollout budget as in-context methods during online training. More detailed descriptions and experiment settings can be found in Appendix~\ref{appendix-kd}. As demonstrated in Table~\ref{tab:kernelbench-benchmark}, PCPO achieves the best average rank, which shows the effectiveness of parametric self-evolution at inference time.

\section{Conclusion}
In this paper, we develop self-evolving algorithm-design agents that evolve via exploring and training on self-generated algorithms with population experiences. Our proposed PCPO improves results across two algorithm-design tasks (i.e., GP in EDA and GPU kernel design) over state-of-the-art in-context evolutionary methods (e.g., OpenEvolve and ShinkaEvolve) remains competitive with frontier closed-source models based on Qwen3-8B, and reduces inference-time token cost, showing the benefit of population-curated parametric self-evolution for algorithm-design agents rather than stagnating with in-context evolution.
\newpage

\bibliographystyle{plain}
\bibliography{references}

@article{lu2026advancing,
  title={Advancing automated algorithm design via evolutionary stagewise design with LLMs},
  author={Lu, Chen and Xue, Ke and Gao, Chengrui and Shi, Yunqi and Xu, Siyuan and Yuan, Mingxuan and Qian, Chao and Zhou, Zhi-Hua},
  journal={arXiv:2603.07970},
  year={2026}
}

@inproceedings{IBM-offpolicy,
  title={Revisiting group relative policy optimization: Insights into on-policy and off-policy training},
  author={Mroueh, Youssef and Dupuis, Nicolas and Belgodere, Brian and Nitsure, Apoorva and Rigotti, Mattia and Greenewald, Kristjan and Navratil, Jiri and Ross, Jarret and Rios, Jesus},
  booktitle={Proceedings of the 14th International Conference on Learning Representations},
  address = {Rio de Janeiro, Brazil},
  year={2026}
}

@article{Meta-offpolicy,
  title={Efficient {RL} training for {LLM}s with experience replay},
  author={Arnal, Charles and Cabannes, Vivien and Cohen, Taco and Kempe, Julia and Munos, Remi},
  journal={arXiv:2604.08706},
  year={2026}
}

@inproceedings{
kernelbench,
title={KernelBench: Can {LLM}s Write Efficient {GPU} Kernels?},
author={Anne Ouyang and Simon Guo and Simran Arora and Alex L Zhang and William Hu and Christopher Re and Azalia Mirhoseini},
booktitle={Proceedings of the 42nd International Conference on Machine Learning},
address = {Vancouver, Canada},
year={2025},
}

@article{yang2025qwen3,
  title={Qwen3 technical report},
  author={Yang, An and Li, Anfeng and Yang, Baosong and Zhang, Beichen and Hui, Binyuan and Zheng, Bo and Yu, Bowen and Gao, Chang and Huang, Chengen and Lv, Chenxu and others},
  journal={arXiv:2505.09388},
  year={2025}
}

@inproceedings{lange2026shinkaevolve,
  title={Shinkaevolve: Towards open-ended and sample-efficient program evolution},
  author={Lange, Robert and Imajuku, Yuki and Cetin, Edoardo},
  booktitle={Proceedings of the 14th International Conference on Learning Representations},
  address = {Rio de Janeiro, Brazil},
  year={2026}
}

@inproceedings{tillet2019triton,
  title={Triton: an intermediate language and compiler for tiled neural network computations},
  author={Tillet, Philippe and Kung, Hsiang-Tsung and Cox, David},
  booktitle={Proceedings of the 3rd ACM SIGPLAN International Workshop on Machine Learning and Programming Languages},
  year={2019},
  address = {Phoenix, AZ}
}

@inproceedings{spector2025thunderkittens,
  title={Thunderkittens: Simple, fast, and adorable kernels},
  author={Spector, Benjamin and Arora, Simran and Singhal, Aaryan and Parthasarathy, Arjun and Fu, Dan and Re, Christopher},
  booktitle={Proceedings of the 13th International Conference on Learning Representations},
  year={2025},
  address = {Singapore}
}

@article{wang2024agentsurvey,
  title={A survey on large language model based autonomous agents},
  author={Wang, Lei and Ma, Chen and Feng, Xueyang and Zhang, Zeyu and Yang, Hao and Zhang, Jingsen and Chen, Zhiyuan and Tang, Jiakai and Chen, Xu and Lin, Yankai and others},
  journal={Frontiers of Computer Science},
  volume={18},
  number={6},
  pages={186345},
  year={2024},
  publisher={Springer}
}

@article{xi2025agentsurvey,
  title={The rise and potential of large language model based agents: A survey},
  author={Xi, Zhiheng and Chen, Wenxiang and Guo, Xin and He, Wei and Ding, Yiwen and Hong, Boyang and Zhang, Ming and Wang, Junzhe and Jin, Senjie and Zhou, Enyu and others},
  journal={Science China information sciences},
  volume={68},
  number={2},
  pages={121101},
  year={2025},
  publisher={Springer}
}

@article{ttt-discover,
  title={Learning to discover at test time},
  author={Yuksekgonul, Mert and Koceja, Daniel and Li, Xinhao and Bianchi, Federico and McCaleb, Jed and Wang, Xiaolong and Kautz, Jan and Choi, Yejin and Zou, James and Guestrin, Carlos and others},
  journal={arXiv:2601.16175},
  year={2026}
}

@inproceedings{wang2026thetaevolve,
title={ThetaEvolve: Test-time learning on open problems},
author={Yiping Wang and Shao-Rong Su and Zhiyuan Zeng and Eva Xu and Liliang Ren and Xinyu Yang and Zeyi Huang and Xuehai He and Luyao Ma and Baolin Peng and Hao Cheng and Pengcheng He and Weizhu Chen and Shuohang Wang and Simon Shaolei Du and Yelong Shen},
booktitle={Proceedings of the 43rd International Conference on Machine Learning},
year={2026},
address={Seoul, South Korea}
}

@inproceedings{jitrl,
title={Just-In-Time Reinforcement Learning: Continual Learning in {LLM} Agents Without Gradient Updates},
author={Li, Yibo and Lin, Zijie and Deng, Ailin and Zhang, Xuan and He, Yufei and Ji, Shuo and Cao, Tri and Hooi, Bryan},
booktitle={Proceedings of the 43rd International Conference on Machine Learning},
year={2026},
address={Seoul, South Korea}
}

@inproceedings{
zhang2026selfevolving-offline-rl,
title={Self-evolving {LLM} agents with in-distribution Optimization},
author={Yudi Zhang and Meng Fang and Zhenfang Chen and Mykola Pechenizkiy},
booktitle={Proceedings of the 43rd International Conference on Machine Learning},
year={2026},
address={Seoul, South Korea}
}

@inproceedings{
zhang2026memevolve,
title={MemEvolve: Meta-evolution of agent memory systems},
author={Guibin Zhang and Haotian Ren and Chong Zhan and  Zhenhong Zhou and Junhao Wang and He Zhu and Wangchunshu Zhou and Shuicheng Yan},
booktitle={Proceedings of the 43rd International Conference on Machine Learning},
year={2026},
address={Seoul, South Korea}
}

@article{team2026gemma,
  title={Gemma 4 technical report},
  author={Team, Gemma and Abd, Sherif El and Aggarwal, Vaibhav and Algayres, Robin and Andreev, Alek and Bachem, Olivier and Ballantyne, Ian and Brick, Cormac and C{\u{a}}rbune, Victor and Casbon, Michelle and others},
  journal={arXiv:2607.02770},
  year={2026}
}

@article{self-evolving-survey-1,
  title={A comprehensive survey of self-evolving {AI} agents: A new paradigm bridging foundation models and lifelong agentic systems},
  author={Fang, Jinyuan and Peng, Yanwen and Zhang, Xi and Wang, Yingxu and Yi, Xinhao and Zhang, Guibin and Xu, Yi and Wu, Bin and Liu, Siwei and Li, Zihao and others},
  journal={arXiv:2508.07407},
  year={2025}
}

@article{self-evolving-survey-2,
title={A survey of self-evolving agents: What, when, how, and where to evolve on the path to artificial super intelligence},
author={Gao, Huan-ang and Geng, Jiayi and Hua, Wenyue and Hu, Mengkang and Juan, Xinzhe and Liu, Hongzhang and Liu, Shilong and Qiu, Jiahao and Qi, Xuan and Wu, Yiran and others},
volume = {2026-January},
journal = {Transactions on Machine Learning Research},
publisher = {Transactions on Machine Learning Research},
year = {2026}
}

@inproceedings{Domain-Knowledge-Injection,
    title = "Structure-aware Domain Knowledge Injection for Large Language Models",
    author = "Liu, Kai  and
      Chen, Ze  and
      Fu, Zhihang  and
      Zhang, Wei  and
      Jiang, Rongxin  and
      Zhou, Fan  and
      Chen, Yaowu  and
      Wu, Yue  and
      Ye, Jieping",
    booktitle = "Proceedings of the 63rd Annual Meeting of the Association for Computational Linguistics",
    year = "2025",
    address = "Vienna, Austria"
}

@inproceedings{guo-etal-2025-efficient-domain,
    title = "Efficient Domain Continual pretraining by Mitigating the Stability Gap",
    author = "Guo, Yiduo  and
      Fu, Jie  and
      Zhang, Huishuai  and
      Zhao, Dongyan",
    booktitle = "Proceedings of the 63rd Annual Meeting of the Association for Computational Linguistics",
    year = {2025},
    address = "Vienna, Austria"
}

@inproceedings{Domain-adaptive-Post-training,
    title = "Demystifying Domain-adaptive Post-training for Financial {LLM}s",
    author = "Ke, Zixuan  and
      Ming, Yifei  and
      Nguyen, Xuan-Phi  and
      Xiong, Caiming  and
      Joty, Shafiq",
    booktitle = "Proceedings of the 30th Conference on Empirical Methods in Natural Language Processing",
    year = "2025",
    address = "Suzhou, China"
}

@inproceedings{sft,
 author={Ouyang, Long and Wu, Jeffrey and Jiang, Xu and Almeida, Diogo and Wainwright, Carroll and Mishkin, Pamela and Zhang, Chong and Agarwal, Sandhini and Slama, Katarina and Ray, Alex and others},
 booktitle = {Advances in Neural Information Processing Systems 36},
 title = {Training language models to follow instructions with human feedback},
 year = {2022},
 address = {New Orleans, LA}
}

@article{sft-survey,
  title={Instruction tuning for large language models: A survey},
  author={Zhang, Shengyu and Dong, Linfeng and Li, Xiaoya and Zhang, Sen and Sun, Xiaofei and Wang, Shuhe and Li, Jiwei and Hu, Runyi and Zhang, Tianwei and Wang, Guoyin and others},
  journal={ACM Computing Surveys},
  volume={58},
  number={7},
  pages={1--36},
  year={2026},
  publisher={ACM New York, NY}
}

@InProceedings{sft-rl,
  title = 	 {{SFT} Memorizes, {RL} Generalizes: A Comparative Study of Foundation Model Post-training},
  author =       {Chu, Tianzhe and Zhai, Yuexiang and Yang, Jihan and Tong, Shengbang and Xie, Saining and Schuurmans, Dale and Le, Quoc V and Levine, Sergey and Ma, Yi},
  booktitle = 	 {Proceedings of the 42nd International Conference on Machine Learning},
  year = 	 {2025},
  address = {Vancouver, Canada}
}

@article{chen2026maxproof,
  title={MaxProof: Scaling Mathematical Proof with Generative-Verifier {RL} and Population-Level Test-Time Scaling},
  author={Chen, Jiacheng and Zhang, Xinyu and Zhang, Shunkai and Wang, Yanmohan and Li, Lin and Qin, Tiancheng and Wang, Qin and Zhu, Zhengmao and Li, Tianle and Li, Jingyang and others},
  journal={arXiv:2606.13473},
  year={2026}
}

@inproceedings{sun2026rl-code,
  title={{RL} Grokking Recipe: How Does {RL} Unlock and Transfer New Algorithms in {LLM}s?},
  author={Sun, Yiyou and Cao, Yuhan and Huang, Pohao and Bai, Haoyue and Hajishirzi, Hanna and Dziri, Nouha and Song, Dawn},
  booktitle={Proceedings of the 14th International Conference on Learning Representations},
  address = {Rio de Janeiro, Brazil},
  year={2026}
}

@article{li2025repo,
  title={Repo: Replay-enhanced policy optimization},
  author={Li, Siheng and Zhou, Zhanhui and Lam, Wai and Yang, Chao and Lu, Chaochao},
  journal={arXiv:2506.09340},
  year={2025}
}

@inproceedings{sun2026cuda,
  title={Cuda-l1: Improving {CUDA} optimization via contrastive reinforcement learning},
  author={Li, Xiaoya and Sun, Xiaofei and Wang, Albert and Li, Jiwei and Shum, Chris},
  booktitle={Proceedings of the 14th International Conference on Learning Representations},
  address = {Rio de Janeiro, Brazil},
  year={2026}
}

@article{kernelbench-verified,
  title={KernelBench-Verified: Do {LLM}-Generated Kernels Actually Beat {P}y{T}orch?},
  author={Zhang, Yunxiang and Yu, Ping and Wang, Jianyu and Fan, Max (Xiangjun) and Reed, Julian and Mirhoseini, Azalia and Su, Will and others},
  journal={arXiv:2607.16241},
  year={2026}
}

@article{context-distillation,
  title={A general language assistant as a laboratory for alignment},
  author={Askell, Amanda and Bai, Yuntao and Chen, Anna and Drain, Dawn and Ganguli, Deep and Henighan, Tom and Jones, Andy and Joseph, Nicholas and Mann, Ben and DasSarma, Nova and others},
  journal={arXiv:2112.00861},
  year={2021}
}

@article{learning-by-distilling-context,
  title={Learning by distilling context},
  author={Snell, Charlie and Klein, Dan and Zhong, Ruiqi},
  journal={arXiv:2209.15189},
  year={2022}
}

@article{evo-plateaus,
  title={Scientific algorithm discovery by augmenting alphaevolve with deep research},
  author={Liu, Gang and Zhu, Yihan and Chen, Jie and Jiang, Meng},
  journal={arXiv:2510.06056},
  year={2025}
}

@article{alphaevolve,
  title={Alpha{E}volve: A coding agent for scientific and algorithmic discovery},
  author={Novikov, Alexander and V{\~u}, Ng{\^a}n and Eisenberger, Marvin and Dupont, Emilien and Huang, Po-Sen and Wagner, Adam Zsolt and Shirobokov, Sergey and Kozlovskii, Borislav and Ruiz, Francisco JR and Mehrabian, Abbas and others},
  journal={arXiv:2506.13131},
  year={2025}
}

@inproceedings{EoH,
title={Evolution of heuristics: Towards efficient automatic algorithm design using large language model},
author={Liu, Fei and Tong, Xialiang and Yuan, Mingxuan and Lin, Xi and Luo, Fu and Wang, Zhenkun and Lu, Zhichao and Zhang, Qingfu},
booktitle={Proceedings of the 41st International Conference on Machine Learning},
year={2024},
address={Vienna, Austria}
}

@inproceedings{ispd2005,
  title={The {ISPD}2005 placement contest and benchmark suite},
  author={Nam, Gi-Joon and Alpert, Charles J and Villarrubia, Paul and Winter, Bruce and Yildiz, Mehmet},
  booktitle={Proceedings of the 2005 International Symposium on Physical Design},
  year={2005},
  address = {San Francisco, CA}
}

@inproceedings{iccad2015,
  author       = {Myung{-}Chul Kim and
                  Jin Hu and
                  Jiajia Li and
                  Natarajan Viswanathan},
  title        = {{ICCAD-2015} {CAD} contest in incremental timing-driven placement and benchmark suite},
  booktitle    = {Proceedings of the 2015 International Conference on Computer-Aided Design},
address={Austin, TX},
  year         = {2015}
}

@inproceedings{ReEvo,
  title={Re{E}vo: Large language models as hyper-heuristics with reflective evolution},
  author={Ye, Haoran and Wang, Jiarui and Cao, Zhiguang and Berto, Federico and Hua, Chuanbo and Kim, Haeyeon and Park, Jinkyoo and Song, Guojie},
  booktitle={Advances in Neural Information Processing Systems 38},
  address={Vancouver, Canada},
  year={2024},
}

@article{FunSearch,
  title={Mathematical discoveries from program search with large language models},
  author={Romera-Paredes, Bernardino and Barekatain, Mohammadamin and Novikov, Alexander and Balog, Matej and Kumar, M Pawan and Dupont, Emilien and Ruiz, Francisco JR and Ellenberg, Jordan S and Wang, Pengming and Fawzi, Omar and others},
  journal={Nature},
  volume={625},
  number={7995},
  pages={468--475},
  year={2024},
  publisher={Nature Publishing Group UK London}
}

@incollection{chu2009placement,
  title={Placement},
  author={Chu, Chris},
  booktitle={Electronic Design Automation},
  pages={635--685},
  year={2009},
  publisher={Elsevier}
}

@inproceedings{markov2012progress,
  title={Progress and challenges in {VLSI} placement research},
  author={Markov, Igor L and Hu, Jin and Kim, Myung-Chul},
  booktitle={Proceedings of the 25th International Conference on Computer-Aided Design},
  year={2012},
  address={San Jose, CA},
}

@article{lin2020dreamplace,
  title={{DREAMP}lace: {D}eep learning toolkit-enabled gpu acceleration for modern {VLSI} placement},
  author={Lin, Yibo and Jiang, Zixuan and Gu, Jiaqi and Li, Wuxi and Dhar, Shounak and Ren, Haoxing and Khailany, Brucek and Pan, David Z},
  journal={IEEE Transactions on Computer-Aided Design of Integrated Circuits and Systems},
  volume={40},
  number={4},
  pages={748--761},
  year={2020},
}

@article{essential-issues-in-analytical,
  author       = {Yao{-}Wen Chang and
                  Zhe{-}Wei Jiang and
                  Tung{-}Chieh Chen},
  title        = {Essential issues in analytical placement algorithms},
  journal={IPSJ Transactions on System LSI Design Methodology},
  volume={2},
  pages={145--166},
  year={2009},
}

@article{chen2008ntuplace3,
  title={{NTU}place3: {A}n analytical placer for large-scale mixed-size designs with preplaced blocks and density constraints},
  author={Chen, Tung-Chieh and Jiang, Zhe-Wei and Hsu, Tien-Chang and Chen, Hsin-Chen and Chang, Yao-Wen},
  journal={IEEE Transactions on Computer-Aided Design of Integrated Circuits and Systems},
  volume={27},
  number={7},
  pages={1228--1240},
  year={2008},
  publisher={IEEE}
}

@article{cheng2018replace,
  title={{R}eplace: {A}dvancing solution quality and routability validation in global placement},
  author={Cheng, Chung-Kuan and Kahng, Andrew B and Kang, Ilgweon and Wang, Lutong},
  journal={IEEE Transactions on Computer-Aided Design of Integrated Circuits and Systems},
  volume={38},
  number={9},
  pages={1717--1730},
  year={2018},
}

@article{lu2015eplace,
  title={e{P}lace: {E}lectrostatics-based placement using fast {F}ourier transform and {N}esterov's method},
  author={Lu, Jingwei and Chen, Pengwen and Chang, Chin-Chih and Sha, Lu and Huang, Dennis Jen-Hsin and Teng, Chin-Chi and Cheng, Chung-Kuan},
  journal={ACM Transactions on Design Automation of Electronic Systems},
  volume={20},
  number={2},
  pages={1--34},
  year={2015},
}

@article{xplace,
  author={Liu, Lixin and Fu, Bangqi and Lin, Shiju and Liu, Jinwei and Young, Evangeline F. Y. and Wong, Martin D. F.},
  journal={IEEE Transactions on Computer-Aided Design of Integrated Circuits and Systems}, 
  title={Xplace: An extremely fast and extensible placement framework}, 
  year={2024},
  volume={43},
  number={6},
  pages={1872-1885},
}

@article{kingma2014adam,
  title={Adam: A method for stochastic optimization},
  author={Kingma, Diederik P and Ba, Jimmy},
  journal={arXiv:1412.6980},
  year={2014}
}

@article{deepseek-r1,
  title={Deepseek-{R}1: Incentivizing reasoning capability in {LLM}s via reinforcement learning},
  author={Guo, Daya and Yang, Dejian and Zhang, Haowei and Song, Junxiao and Zhang, Ruoyu and Xu, Runxin and Zhu, Qihao and Ma, Shirong and Wang, Peiyi and Bi, Xiao and others},
  journal={arXiv:2501.12948},
  year={2025}
}

@inproceedings{agnesina2023autodmp,
  title={Autodmp: Automated dreamplace-based macro placement},
  author={Agnesina, Anthony and Rajvanshi, Puranjay and Yang, Tian and Pradipta, Geraldo and Jiao, Austin and Keller, Ben and Khailany, Brucek and Ren, Haoxing},
  booktitle={Proceedings of the 2023 International Symposium on Physical Design},
  address={Virtual Event},
  year={2023}
}

@misc{whitepaper,
  title={Introduction to Agents},
  author={Blount, Alan and Gulli, Antonio and Saboo, Shubham and Zimmermann, Michael and Vuskovic, Vladimir},
  year={2025},
  publisher={Google Cloud},
  howpublished = {Whitepaper, Google},
  url = {https://www.kaggle.com/whitepaper-introduction-to-agents}
}

\newpage

\appendix
\setcounter{proposition}{0}
\setcounter{corollary}{0}

\section{Proof}\label{appendix-proof}
\begin{proposition}[Local likelihood transfer]
\label{lem:ic-spillover-proof}
Under Assumptions~\ref{assump:ic-edit-lipschitz}, for every
$i\in\{0,\ldots,n-1\}$,
\begin{equation}
    \ell_{\theta_{i+1}}(A_{i+1}\mid C)
    -\ell_{\theta_i}(A_{i+1}\mid C)
    \geq
    \epsilon_\ell-2\kappa d(A_i,A_{i+1})
    \geq
    \epsilon_\ell-2\kappa\Delta.
\end{equation}
Consequently, if $\rho_\ell\triangleq\epsilon_\ell-2\kappa\Delta>0$, then
\begin{equation}
    p_{\theta_{i+1}}(A_{i+1}\mid C)
    \geq
    e^{\rho_\ell}p_{\theta_i}(A_{i+1}\mid C).
\end{equation}
\end{proposition}

\begin{proof}
Let $\theta=\theta_i$ and $\theta'=\theta_{i+1}$. By Assumption~\ref{assump:ic-edit-lipschitz},
\begin{align}
&\ell_{\theta'}(A_{i+1}\mid C)-\ell_\theta(A_{i+1}\mid C) \nonumber\\
&=\bigl[\ell_{\theta'}(A_{i+1}\mid C)-\ell_{\theta'}(A_i\mid C)\bigr]
 +\bigl[\ell_{\theta'}(A_i\mid C)-\ell_\theta(A_i\mid C)\bigr] \nonumber\\
&\quad
 +\bigl[\ell_\theta(A_i\mid C)-\ell_\theta(A_{i+1}\mid C)\bigr] \nonumber\\
&\geq -\kappa d(A_i,A_{i+1})+\epsilon_\ell-\kappa d(A_i,A_{i+1}),
\end{align}
which proves Eq.~\ref{eq:ic-log-transfer}. Exponentiating both sides gives Eq.~\ref{eq:ic-multiplicative-transfer}.
\end{proof}

\section{Global Placement}

\begin{table}[t]
\centering
\caption{Hyperparameters of PCPO with Qwen3-8B on Global Placement.
Every \(T\) on-policy steps, we additionally run an off-policy update on the global-population top-\(k\) programs ranked by composite score \(S=\alpha\,\tilde{r}+\beta\,d\), where \(\tilde{r}\) is min--max normalized reward and \(d\) is semantic diversity.
Near-duplicates with diversity \(\le\delta_{\max}\) are merged unless the fitness gain is at least \(\gamma_{\min}\).}
\label{tab:hyperparameters}
\small
\setlength{\tabcolsep}{4pt}
\begin{tabular}{lcc}
\toprule
\textbf{Setting} & \textbf{Symbol / name} & \textbf{Value} \\
\midrule
\multicolumn{3}{l}{\textit{Reward and legality}} \\
Illegal / crash / missing API & --- & \(-10^{9}\) \\
Not achieving target overflow & --- & \(-1000\) \\
Valid placement & \(r\) & DREAMPlace \texttt{get\_fitness} (higher better) \\
Illegal threshold (crash band) & \(r_{\mathrm{inv}}\) & \(-10^{8}\) \\
\midrule
\multicolumn{3}{l}{\textit{PCPO}} \\
Off-policy interval & \(T\) & \(2\) \\
Population top-\(k\) & \(k\) & \(4\) \\
Composite reward weight & \(\alpha\) & \(1.0\) \\
Composite diversity weight & \(\beta\) & \(1.0\) \\
Near-duplicate diversity cutoff & \(\delta_{\max}\) & \(0.15\) \\
Min.\ fitness gain to replace near-dup. & \(\gamma_{\min}\) & \(0.01\) \\
Replay buffer size & --- & \(512\) \\
On-policy learning rate & \(lr_\text{on}\) & \(1\times 10^{-6}\) \\
Off-policy learning rate & \(lr_\text{off}\) & \(2\times 10^{-6}\) \\
\midrule
\multicolumn{3}{l}{\textit{Optimization}} \\
Train batch size (prompts) & \(B\) & \(4\) \\
Group size (GRPO) & \(n\) & \(4\) \\
PPO epochs & --- & \(1\) \\
KL coefficient (low-var.\ KL) & \(\beta_{\mathrm{KL}}\) & \(0.001\) \\
Clip ratio / entropy / grad.\ clip & --- & \(0.2\) / \(0.001\) / \(1.0\) \\
Training epochs / steps & --- & \(1\) / \(16\) \\
Dataset size & --- & \(64\) prompts \\
Max prompt / response tokens & --- & \(2048\) / \(2048\) \\
\midrule
\multicolumn{3}{l}{\textit{Decoding}} \\
Temperature / top-\(p\) & --- & \(0.5\) / \(1.0\)\\
Max new tokens & --- & \(1024\) \\
Qwen3 thinking mode & --- & off \\
\bottomrule
\end{tabular}
\end{table}

\subsection{Introduction to Global Placement}\label{appendix-gp-intro}
Global placement (GP) is a critical step of the overall chip placement procedure~\citep{chu2009placement}, which has a direct and significant impact on the performance, power consumption, and area (PPA) of the final chip design, as it determines the locations of the cells within the given chip layout area~\citep{markov2012progress}. GP aims to deliver a high‑quality rough overall layout for all cells and is required to meet the target overflow, a metric that measures a chip’s total overlap. It should also achieve the lowest possible half‑perimeter wirelength (HPWL) and run as fast as possible.~\citep{lin2020dreamplace}. For formulation, GP is typically defined as a constrained optimization problem. It uses HPWL as the objective and density as the constraint.: 
\begin{equation}
\begin{gathered}
    \underset{\textbf{x}, \textbf{y}}{\min}\text{ HPWL}(\textbf{x},\textbf{y})=\underset{\textbf{x}, \textbf{y}}{\min}\sum\nolimits_{e\in E}\text{HPWL}_e(\textbf{x},\textbf{y}),\\
\text{s.t. } \text{D}(\textbf{x}, \textbf{y}) \le d_t,
\end{gathered}
\label{eq-gp-constrained}
\end{equation}
where $(\textbf{x}, \textbf{y})$ denotes the 2-D positions of all the cells, $E$ denotes the set of nets,  $\text{HPWL}_e(\textbf{x},\textbf{y})= (\text{max}_{i\in e} x_i - \text{min}_{i\in e} x_i) + (\text{max}_{i\in e} y_i - \text{min}_{i\in e} y_i)$ denotes the HPWL of each net $e$, $\text{D}(\textbf{x}, \textbf{y})$ denotes the density of each location of the layout, and $d_t$ denotes the target density. Note that a net covers multiple cells and represents how these cells connect with one another during the routing phase.

\textbf{Existing GP methods:} 
Analytical placement~\citep{essential-issues-in-analytical} represents one class of state‑of‑the‑art GP methods. Among these approaches, nonlinear placement~\citep{chen2008ntuplace3, lu2015eplace, cheng2018replace} delivers the best performance. Nonlinear placement generally addresses the constrained GP optimization problem defined in Eq.~(\ref{eq-gp-constrained}) using the penalty method. Specifically, it solves a series of Lagrangian‑relaxation problems with progressively increasing Lagrangian multipliers, as formulated below:
\begin{equation}
    \underset{\textbf{x}, \textbf{y}}{\text{min}}\;\left(\sum\nolimits_{e\in E} \text{WL}_e(\textbf{x},\textbf{y})\right) + \lambda \cdot  \text{D}(\textbf{x}, \textbf{y}),
    \label{target}
\end{equation}
where $\text{WL}_e(\textbf{x},\textbf{y})$ is a smooth approximation of $\text{HPWL}_e(\textbf{x},\textbf{y})$, and $\lambda$ is the Lagrangian multiplier (also known as density weight), which controls the weight of the density penalty function $\text{D}(\textbf{x}, \textbf{y})$. To achieve smoother optimization, ePlace~\citep{lu2015eplace} and RePlace~\citep{cheng2018replace} calculate a density penalty function analogous to the potential energy in an electrostatic system. In this model, cells act as charges, and the density gradient represents the electric field. This penalty function is smoother than the original density function and maintains strong correlation with it.

Despite their outstanding performance, nonlinear methods are typically time‑consuming. For this reason, they cannot satisfy the fast‑verification requirements of industrial scenarios, particularly for large‑scale chip cases~\citep{lin2020dreamplace}. To address this issue, acceleration algorithms such as DREAMPlace~\citep{lin2020dreamplace} and Xplace~\citep{xplace} have been developed. These algorithms leverage modern GPUs and deep‑learning toolkits to speed up the gradient‑descent optimization process. Specifically, DREAMPlace~\citep{lin2020dreamplace} incorporates modern gradient‑descent optimizers (e.g., Adam, Nesterov, etc.) and adapts neural‑network training techniques to placement tasks. It delivers strong performance and stands as one of the most widely used and advanced open‑source placers available today. However, as shown empirically, the learning rate schedules have a huge impact on the quality and convergence of the GP process and are hard to tune~\citep{agnesina2023autodmp}. Designing the learning rate schedules requires OOD domain knowledge since they need to balance the wire-length and cell density via changing the positions of cells through gradient backpropagation~\citep{lin2020dreamplace}. For the ISPD 2005 benchmark, the target overflow is 0.07; for the ICCAD 2015 benchmark, the target overflow is 0.1.

%
%

\definecolor{cotbg}{HTML}{F7F4EE}
\definecolor{cotframe}{HTML}{5C4A32}
\definecolor{h1c}{HTML}{B42318}
\definecolor{h2c}{HTML}{B54708}
\definecolor{h3c}{HTML}{6927DA}
\definecolor{okc}{HTML}{175CD3}

\newcommand{\Hmark}[2]{%
  \textcolor{#1}{\textbf{[#2]}}%
}

\lstdefinestyle{pytiny}{
  language=Python,
  basicstyle=\ttfamily\fontsize{7.2}{8.8}\selectfont,
  keywordstyle=\bfseries,
  commentstyle=\itshape\color{black!45},
  showstringspaces=false,
  breaklines=true,
  breakatwhitespace=true,
  columns=fullflexible,
  keepspaces=true,
  tabsize=4,
  frame=single,
  rulecolor=\color{black!25},
  framerule=0.3pt,
  xleftmargin=3pt,
  xrightmargin=3pt,
  aboveskip=3pt,
  belowskip=2pt,
  morecomment=[l]{\#},
}

\begin{figure*}[t]
\centering

\begin{tcolorbox}[
  enhanced,
  width=\textwidth,
  colback=cotbg,
  colframe=cotframe,
  boxrule=0.6pt,
  arc=2pt,
  left=6pt, right=6pt, top=5pt, bottom=5pt,
  toptitle=2pt, bottomtitle=2pt,
  fonttitle=\sffamily\bfseries\footnotesize,
  coltitle=white,
  colbacktitle=cotframe,
  title={Qwen3-8B chain-of-thought (OpenEvolve on \texttt{bigblue1}, parent fitness $-1000$)},
  borderline west={3.2pt}{0pt}{h1c},
]
\sffamily\fontsize{8}{10.2}\selectfont
\textit{The current program is using a baseline exponential decay \ldots But the fitness is still $-1000$, which means the overflow is above $0.07$.}

\textit{Looking at the previous attempts, \Hmark{h1c}{H1} the problem seems to be that the learning rate isn't decreasing enough when overflow is high.}
\textit{The current code reduces the learning rate by $0.99$ when overflow is above $0.07$, but maybe that's not enough.}
\textit{\ldots multiplying by $0.95$ instead of $0.99$ when overflow is too high. That would make the learning rate drop faster, which might help in reducing overflow more quickly.}

\vspace{3pt}
\textit{Wait, the \texttt{log\_hpwl\_prev} and \texttt{log\_hpwl} could indicate if the hpwl is improving.}
\textit{\Hmark{h2c}{H2} If \texttt{log\_hpwl} is decreasing (meaning hpwl is getting worse), maybe we need to adjust the learning rate.}
\textit{But in the current code, if \texttt{log\_hpwl} is less than \texttt{log\_hpwl\_prev}, it's considered an improvement, so it increases the learning rate.}

\vspace{3pt}
\textit{\Hmark{h3c}{H3} if \texttt{log\_gradient\_norm} $>$ \texttt{log\_gradient\_norm\_prev}: reduce the learning rate.}
\textit{Wait, but \texttt{log\_gradient\_norm\_prev} is not a parameter. \ldots So, perhaps using the current \texttt{log\_gradient\_norm} as a signal. For example, if \texttt{log\_gradient\_norm} is higher than some threshold, reduce the learning rate.}

\vspace{4pt}
{\scriptsize\color{black!55}
H1:~inverted causality (overflow stays high $\Rightarrow$ decay even faster, compounding on \texttt{learning\_rate\_prev}).
\quad
H2:~inverted semantics ($\downarrow\log\mathrm{HPWL}$ means HPWL \emph{improves}, not worsens).
\quad
H3:~phantom variable (the API has no previous gradient norm).}
\end{tcolorbox}

\vspace{5pt}

\begin{minipage}[t]{0.48\textwidth}
\centering
\textbf{\small (b)~OpenEvolve child of (a)}\par
{\scriptsize fitness $-1000$; final overflow $0.86$ (illegal)}
\begin{lstlisting}[style=pytiny]
def adjust_learning_rate(
    init_learning_rate, step_num, log_hpwl,
    log_hpwl_prev, overflow, log_lambda,
    learning_rate_prev, log_gradient_norm):
    lr = learning_rate_prev * 0.995
    if overflow > 0.07:
        lr *= 0.95          # extra compounding decay
    elif log_hpwl < log_hpwl_prev:
        lr *= 1.002         # HPWL down => raise LR
    elif log_gradient_norm > 5.0:
        lr *= 0.99          # invented threshold
    else:
        lr *= 0.995
    return lr
\end{lstlisting}
\end{minipage}
\hfill
\begin{minipage}[t]{0.50\textwidth}
\centering
\textbf{\small (c)~Direct sample (same task; best of 64)}\par
{\scriptsize fitness $-87.26$; final overflow $0.068$ (legal)}
\begin{lstlisting}[style=pytiny]
def adjust_learning_rate(
    init_learning_rate, step_num, log_hpwl,
    log_hpwl_prev, overflow, log_lambda,
    learning_rate_prev, log_gradient_norm):
    decay = 0.995 ** step_num          # scheduled, not recursive
    hpwl_f = 1.0 + 0.5 * (log_hpwl - log_hpwl_prev)
    ov_f   = 1.0 - 0.5 * max(overflow - 0.07, 0.0)
    g_f    = 1.0 / (1.0 + 0.5 * np.exp(log_gradient_norm))
    lam_f  = 1.0 / (1.0 + 0.2 * np.exp(log_lambda))
    return (init_learning_rate * decay
            * hpwl_f * ov_f * g_f * lam_f)
\end{lstlisting}
\end{minipage}

\caption{A representative OpenEvolve thinking trace on \texttt{bigblue1} versus a strong direct generation based on Qwen3-8B.
Panel~(a) is lightly abridged from the model's \texttt{<think>} block;
highlighted claims H1--H3 are factually false or internally inconsistent.
Panel~(b) is the program emitted from that trace: extra multiplicative decay on \texttt{learning\_rate\_prev} while overflow exceeds $0.07$ compounds into $\eta_t=\eta_0(0.995\cdot 0.95)^t$, the optimizer stalls, and overflow never reaches the $0.07$ target.
Panel~(c) reconstructs the expert schedule $\eta_0\cdot 0.995^{t}$ and applies non-compounding state-dependent modulators; HPWL improving (\texttt{log\_hpwl}${}<{}$\texttt{log\_hpwl\_prev}) \emph{shrinks} the step, the opposite of~(b).
Ellipses in~(a) mark omitted sentences; wording is otherwise verbatim.}
\label{fig:bigblue1-cot-vs-direct}
\end{figure*}


\begin{table}[t]
\centering
\caption{Recurring errors in OpenEvolve chain-of-thought on \texttt{bigblue1}.
``Direct'' denotes the best sample produced by the base model directly from the same task (Fig.~\ref{fig:bigblue1-cot-vs-direct}c).}
\label{tab:bigblue1-cot-errors}
\small
\setlength{\tabcolsep}{3.5pt}
\begin{tabular}{llp{0.46\columnwidth}}
\toprule
ID & Type & Typical claim in \texttt{<think>} vs.\ the working policy \\
\midrule
H1 & Causal inversion
  & Overflow remains $\approx 0.65$ $\Rightarrow$ ``LR is not small enough'';
    extra $\times 0.99/0.95/0.9$ on \texttt{learning\_rate\_prev}.
    Direct: non-compounding soft scale on $\eta_0\cdot 0.995^{t}$. \\
H2 & Semantic inversion
  & ``If \texttt{log\_hpwl} is decreasing (meaning hpwl is getting worse)''.
    $\log$ is monotone: a drop is an \emph{improvement}. \\
H3 & Phantom input
  & Invents \texttt{log\_gradient\_norm\_prev}, which is not in the API. \\
H4 & Incommensurable compare
  & \texttt{log\_gradient\_norm < log\_hpwl\_prev - 0.01} (iter.~11);
    gradient log vs.\ wirelength log. \\
H5 & Fitness polarity
  & Treats $-88.90$ / $-89.26$ as ``better than baseline'' $-88.56$
    (higher is better). \\
\bottomrule
\end{tabular}
\end{table}


\begin{figure}[t]
\centering
\begin{tcolorbox}[
  enhanced,
  width=\columnwidth,
  colback=cotbg,
  colframe=cotframe,
  boxrule=0.6pt,
  arc=2pt,
  left=5pt, right=5pt, top=4pt, bottom=4pt,
  fonttitle=\sffamily\bfseries\footnotesize,
  coltitle=white,
  colbacktitle=cotframe,
  title={Same run: fitness polarity + incommensurable compare},
  borderline west={3.2pt}{0pt}{h3c},
]
\sffamily\fontsize{8}{10}\selectfont
\textit{Program~3 (the current one) has a fitness of $-89.2598$, which is \Hmark{h1c}{H5} better than the baseline.}
\textit{The overflow is already under $0.07$, so maybe the learning rate is too low \ldots}

\vspace{3pt}
\textit{In Program~3 they added a gradient check:}
\textit{\Hmark{h3c}{H4} if the gradient norm is decreasing (\texttt{log\_gradient\_norm} $<$ \texttt{log\_hpwl\_prev} $-$ $0.01$), they increase the rate.}
\end{tcolorbox}

\vspace{3pt}
\begin{lstlisting}[style=pytiny, caption={Emitted body (fitness $-89.04$; legal but worse than baseline $-88.56$).},
                  label={lst:oe-iter11}]
lr = learning_rate_prev * 0.995
if overflow > 0.07:
    pass
else:
    if log_hpwl < log_hpwl_prev:
        lr *= 1.005
    else:
        lr *= 0.995
    # Increase LR slightly when gradient is decreasing
    if log_gradient_norm < log_hpwl_prev - 0.01:
        lr *= 1.002
return lr
\end{lstlisting}

\caption{A second OpenEvolve trace on the legal island.
H5 inverts fitness polarity ($-89.26$ is \emph{worse} than baseline $-88.56$).
H4 compares $\log\|\nabla\|$ to $\log\mathrm{HPWL}$, two unrelated quantities, then treats the predicate as ``gradient is decreasing.''
The model later notes that the previous gradient is unavailable, yet still emits the comparison.}
\label{fig:bigblue1-incommensurable}
\end{figure}

\subsection{Hyperparameters}
We provide the training hyperparameters of PCPO in Table~\ref{tab:hyperparameters}. All RL methods have the same rollout budget in the training process.

\subsection{Extended results}
\subsubsection{Practical examples for diagnostic experiments of in-context evolutionary stagnation} \label{appendix-stagnation}
On bigblue1, we inspect OpenEvolve's thinking traces with the base model, Qwen3-8B.
A representative excerpt is shown in Fig.~\ref{fig:bigblue1-cot-vs-direct}a.
The model does not merely mistune constants; it searches a
qualitatively inverted learning-rate policy relative to a strong
direct sample from the same model family (Fig.~\ref{fig:bigblue1-cot-vs-direct}c,
fitness $-87.26$, overflow $0.068$).

Two errors recur in OpenEvolve’s thinking traces (shown in Table~\ref{tab:bigblue1-cot-errors}).
\textbf{(H1)~Causal inversion:}
Residual overflow ($\approx 0.65$) is attributed to an insufficiently
small step size, so the child multiplies an extra $0.99$--$0.9$
onto \texttt{learning\_rate\_prev}.
Because the extra factor compounds, the realized schedule is
$\eta_t=\eta_0(0.995\cdot 0.95)^t$ rather than the expert
$\eta_0\cdot 0.995^{t}$; the placer stalls and every such child scores
the hard penalty $-1000$.
The direct program instead recomputes $\eta_0\cdot 0.995^{t}$
each step and applies a one-shot soft scale
$1-0.5\max(\mathrm{overflow}-0.07,0)$, which does not accumulate.
\textbf{(H2)~Semantic inversion.}
The trace states that a decreasing \texttt{log\_hpwl} means HPWL is "getting worse", contradicting monotonicity of $\log$ and the model's
own later clause that treats the same predicate as improvement. For (H3) Phantom input, although in this case the model realizes that \texttt{log\_gradient\_norm\_prev} is not a parameter and avoids using it, the model produces some invalid codes with the phantom input in other cases with a great tendency.

A second trace (Fig.~\ref{fig:bigblue1-incommensurable}) additionally
inverts fitness polarity ($-89.26$ described as ``better'' than the baseline $-88.56$) and compares incommensurable logs,
\texttt{log\_gradient\_norm < log\_hpwl\_prev - 0.01}.
Together these traces explain why evolution never recovers the direct
policy: the search is conducted inside a misleading reasoning, not around the correct schedule.

\begin{figure}[t]
\centering
\centering
\includegraphics[width=\textwidth]{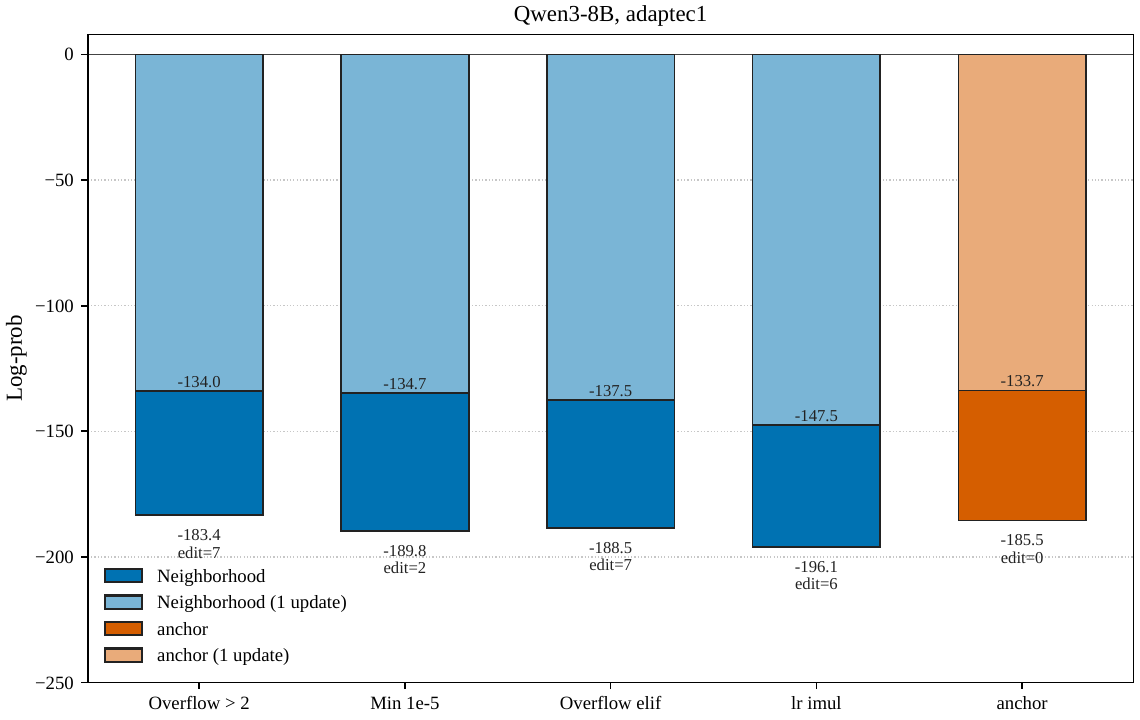} 
\caption{Joint log-probability assigned by Qwen3-8B to five global placement learning-rate policies under a shared prompt, before and after one training update step on the anchor program. The orange bar is the log-probability of the anchor produced by PCPO; blue bars are the log-probabilities of short-edit neighbors with respect to the anchor; bars with lighter colors show the log-probability after one training update step. Token-level edit distance from the anchor is annotated on each bar.}
\label{fig:logprob}
\end{figure}

\begin{figure}[t]
\centering
\centering
\includegraphics[width=\textwidth]{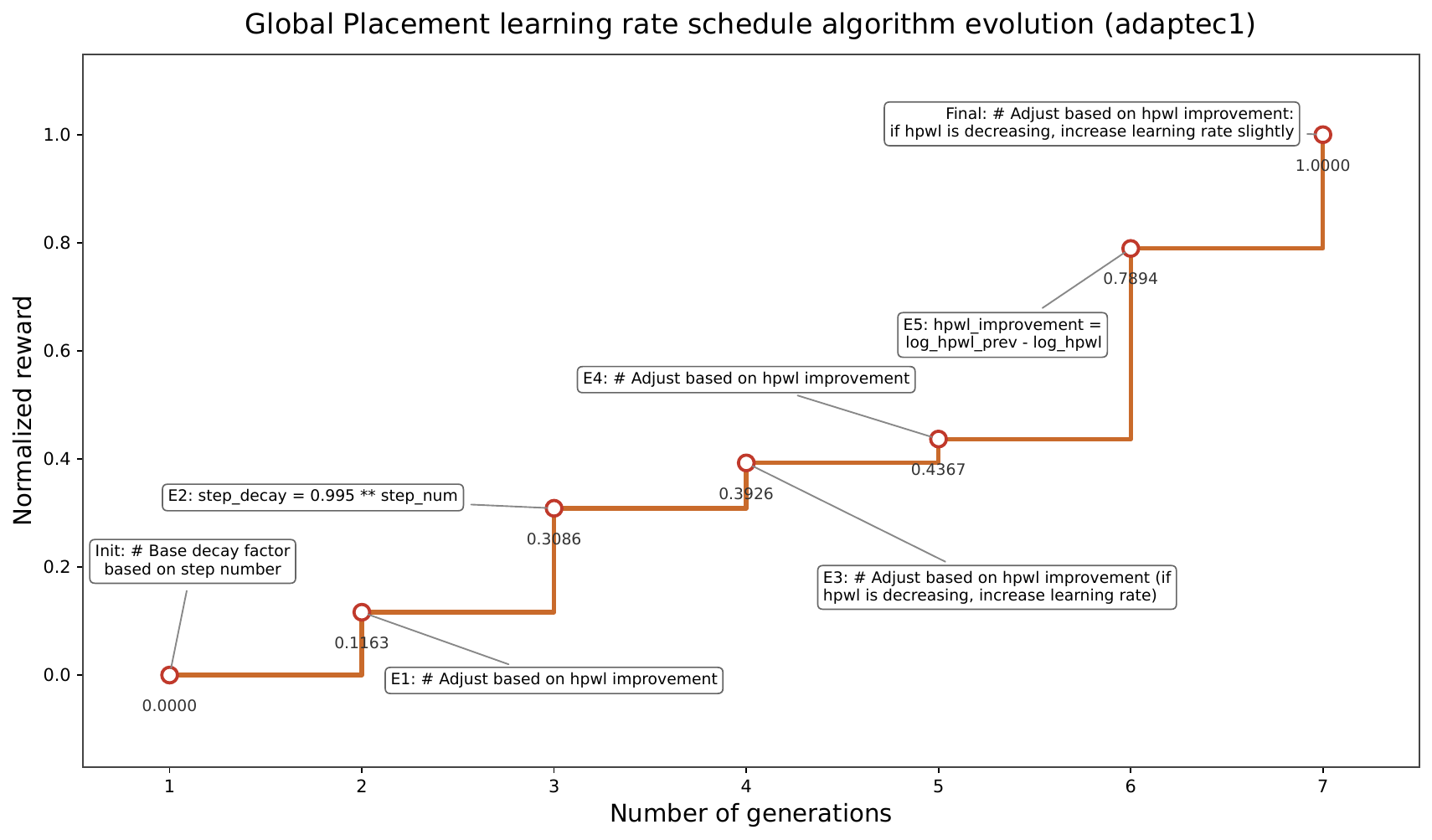} 
\caption{Detailed modifications and improvement of the global placement learning rate schedule algorithm evolution made by PCPO on adaptec1.}
\label{fig:evolution}
\end{figure}

\subsubsection{Practical examples for the Improvement Chain} \label{appendix-gp-case}
To provide a sense of the actual magnitude of the terms $\kappa\Delta$ and $\epsilon_\ell$, we first score Qwen3-8B on five complete adjust\_learning\_rate implementations for global placement (benchmark adaptec1). All candidates are under the same full task prompt. The anchor $A_i$ is a program produced by PCPO, and four neighborhood variants are obtained by short edits to this program. We also score Qwen3-8B on these programs after one training update step on the anchor program with a 2e-6 learning rate (the learning rate used for off-policy update) and an advantage of 1. In this case, we can set $\Delta=10$ since the largest edit distance here is 7. As shown in Figure~\ref{fig:logprob}, the log-probabilities of neighborhood variants remain in a neighborhood range, and thus we can set $\kappa=2.5$ because the maximum value of $\left|\ell_{\theta}(A_i\mid C)-\ell_{\theta}(A_{i+1}\mid C)\right|/d(A_i, A_{i+1})$ before and after the update step is 2.3 (i.e., $(147.5-133.7)/6=2.3$). The log-probability of the anchor program increases by 51.8 (from -185.5 to -133.7) after one training update step on it, and thus we can set $\epsilon_\ell=51$ here. The figure therefore provides a limited sanity check for Assumption~\ref{assump:ic-edit-lipschitz} and Proposition~\ref{lem:ic-spillover} at a single model checkpoint, showing that, in this case, $\rho_\ell=\epsilon_\ell-2\kappa\Delta=51-2\times2.5\times10>0$. 

Moreover, Figure~\ref{fig:evolution} shows the detailed modifications and the corresponding improvements of the global placement learning rate schedule algorithm evolution made by PCPO during the training process on adaptec1. We can observe that the modifications and improvements are consistent with the proposition of the Improvement Chain described in Section~\ref{IC}, where successively discovered algorithms differ through relatively small modifications and achieve progressively better rewards.

\begin{table}[t]
  \centering
  \caption{Average inference-time 64-rollout best HPWL ($\times 10^6$) of 4 independent runs on four global placement cases. Lower is better.}
  \label{tab:gp-inference-time}
  \small
  \begin{tabular}{lcccc|c}
    \toprule
    Method & adaptec1 & bigblue1 & superblue1 & superblue7 & Avg rank \\
    \midrule
    OpenEvolve   & 70.99          & 87.53          & 389.18          & 554.48          & 4.25 \\
    ShinkaEvolve & 71.10          & 88.29          & 389.18          & 551.10          & 4.63 \\
    JitRL        & 70.99          & 88.36          & 387.61          & 556.13          & 4.38 \\
    ThetaEvolve  & 70.68          & 87.83          & 388.79          & 562.56          & 4.25 \\
    GRPO         & 70.49          & 87.51          & 388.09          & 551.60          & 2.50 \\
    PCPO         & \textbf{70.04} & \textbf{87.23} & \textbf{385.64} & \textbf{547.11} & \textbf{1.00} \\
    \bottomrule
  \end{tabular}
\end{table}
\begin{table}[t]
  \centering
  \caption{Average inference-time 64-rollout best HPWL ($\times 10^6$) of 4 independent runs produced by PCPO versus Qwen3-8B equipped with a generalizable skill distilled by GPT-6-Astra from PCPO's training logs. Lower is better.}
  \label{tab:skill-vs-pcpo}
  \small
  \begin{tabular}{lccccc}
    \toprule
    Method & adaptec1 & bigblue1 & superblue1 & superblue7 & Average \\
    \midrule
    PCPO   & \textbf{70.04} & \textbf{87.23} & \textbf{385.64} & \textbf{547.11} & -- \\
    Qwen3-8B + Skill & 70.94          & 87.58          & 388.76          & 553.16& -- \\
    \midrule
    Relative\ gap (\%)   & $+1.29$        & $+0.40$        & $+0.81$        & $+1.11$        & $+0.90$ \\
    \bottomrule
  \end{tabular}\\[2pt]
\end{table}

\subsubsection{Online performance comparison}\label{appedic-inference-time-exp}
Table~\ref{tab:gp-inference-time} shows the online performance comparison of OpenEvolve, ShinkaEvolve, ThetaEvolve, GRPO, and PCPO. Since ThetaEvolve is a test-time learning method that enhances evolution by RL, we adopt PCPO and GRPO in the same setting with a total online training rollout budget of 64. For in-context evolutionary methods, OpenEvolve and ShinkaEvolve, they have the same rollout budget of 64. As shown in Table~\ref{tab:gp-inference-time}, PCPO achieves the best performance, outperforming ThetaEvolve and demonstrating the effectiveness of reusing high-quality population experiences rather than simply discarding each sample after one update via traditional RL. PCPO also beats in-context evolutionary methods in an online setting, dispelling doubts about unfairness in the training of PCPO in the main experiment. Note that ThetaEvolve does not make much improvement compared to OpenEvolve, which may be attributed to repeatedly putting flawed self-generated algorithms into prompts, reinforcing the low-quality generation pattern via RL described in Section~\ref{sec:in-context stagnation}. We also make comparisons to JitRL~\citep{jitrl}, a training-free framework that enables test-time policy optimization by estimating action advantages from the memory and directly modulating the LLM’s output logits without any gradient updates. As shown in Table~\ref{tab:gp-inference-time}, JitRL does not achieve great results since it does not drift the policy through gradient updates: only estimating action advantages from the memory still introduces the model's own bias. 

We also conduct an extra experiment to provide evidence for the effectiveness of parametric evolution over in-context instruction in Table~\ref{tab:skill-vs-pcpo}. We use the most frontier closed-source model, GPT-6-Astra, to analyze and distill a generalizable skill from the training logs of PCPO and equip Qwen3-8B with it for the same budget of 64-rollout generation. As shown in Table~\ref{tab:skill-vs-pcpo}, PCPO still outperforms the base model Qwen3-8B equipped with the skill distilled by GPT-6-Astra from the training logs of PCPO. It shows that even given the same amount of extra information obtained independently by PCPO, in-context methods (e.g., quipping the base model with skills) still cannot make full use of the experiences, stressing the importance of parametric self-evolution.

\begin{figure}[t]
\centering
\centering
\includegraphics[width=0.8\textwidth]{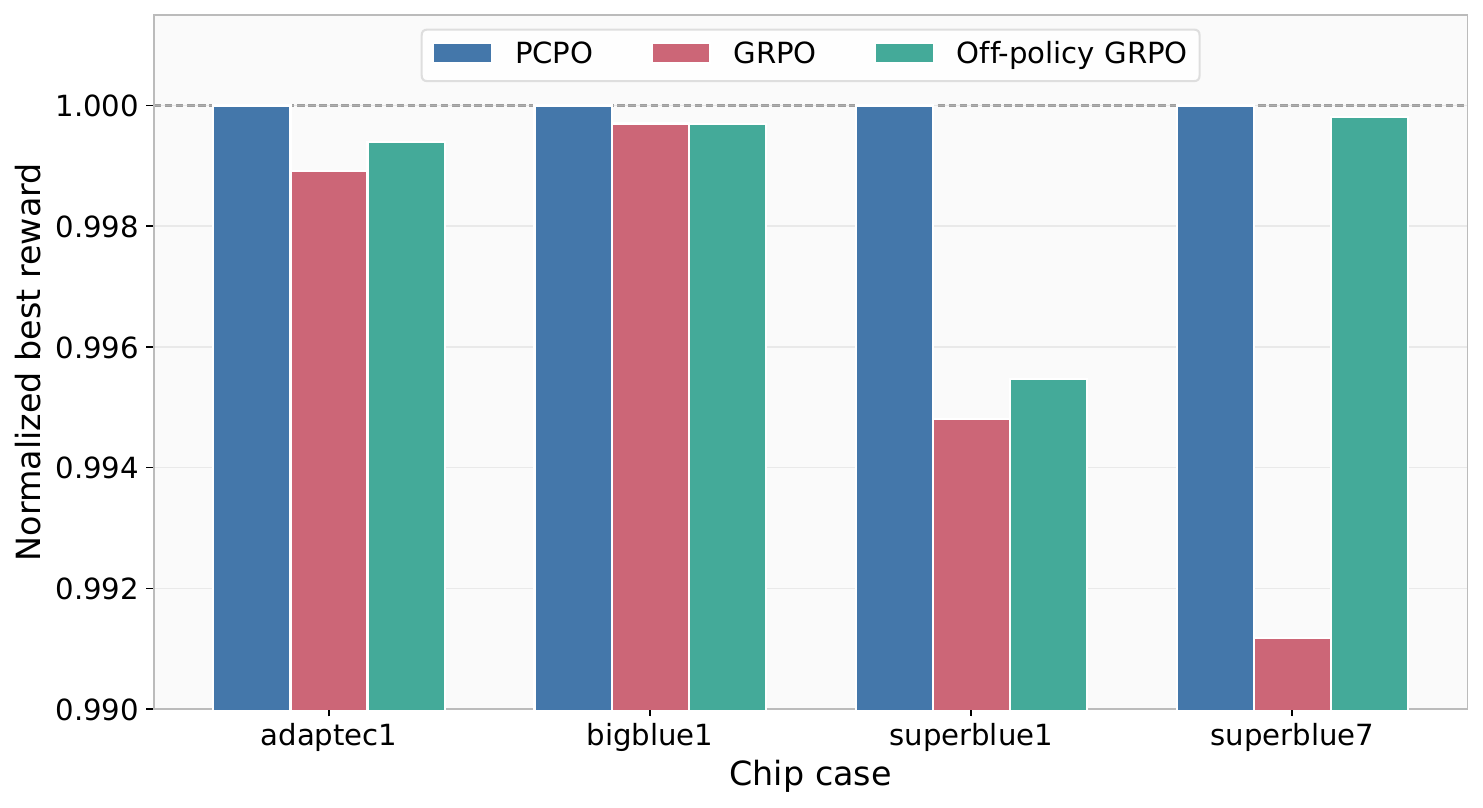} 
\caption{Ablation study on the training performance of the Regularized Off-policy Update and the Refreshing On-policy Update in the Hybrid Policy Update module of PCPO.}
\label{fig:ablation}
\end{figure}

\begin{figure}[t]
\centering
\centering
\includegraphics[width=\textwidth]{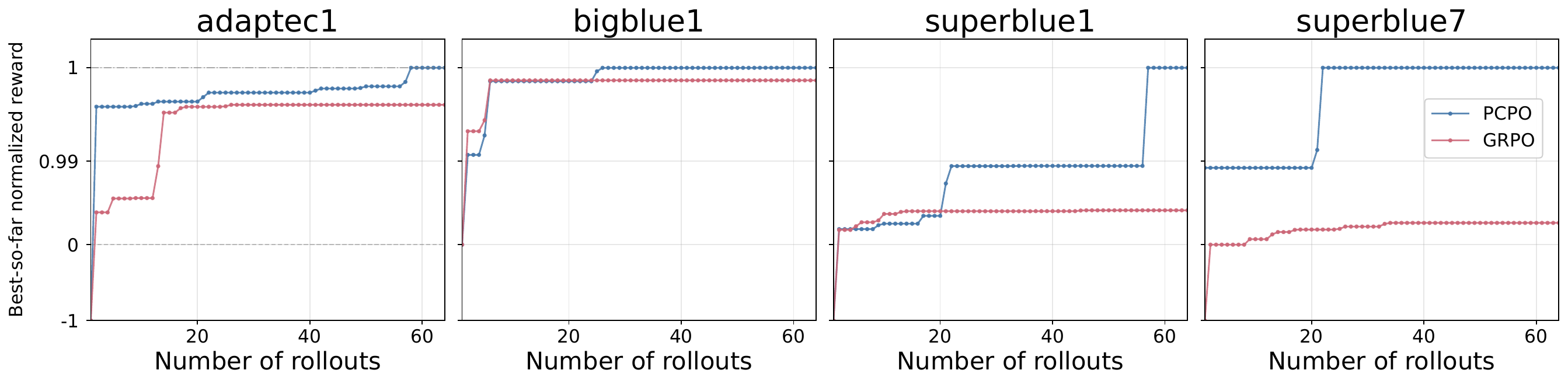} 
\caption{Best-so-far normalized reward along the training trajectories of PCPO and GRPO on 4 training cases.}
\label{fig:mix_training}
\end{figure}

\begin{table*}[t]
  \centering
  \caption{Ablation study of the Global Population: best@n HPWL ($\times 10^6$) results achieved by RePO, Random Replay, and PCPO. The lower, the better.}
  \label{tab:repo}
  \small
  \setlength{\tabcolsep}{3.5pt}
  \begin{adjustbox}{max width=\textwidth}
  \begin{tabular}{llccccccccc}
    \toprule
    Method & Metric
      & adaptec1 & adaptec2 & adaptec3 & adaptec4
      & bigblue1 & bigblue2 & bigblue3 & bigblue4 & Avg Rank \\
    \midrule
    \multirow{3}{*}{RePO}
      & best@1  & $\times$ & 81.39 & \textbf{186.76} & 170.45 & 143.39 & $\times$ & $\times$ & \textbf{737.62} & 2.19 \\
      & best@4  & 71.73 & 80.34 & 186.76 & 170.45 & 88.61 & \textbf{132.71} & 389.33 & \textbf{737.62} & 2.25 \\
      & best@16 & 71.04 & 79.90 & 186.76 & 169.63 & \textbf{87.50} & 132.34 & 309.41 & 733.31 & 2.13 \\
    \cmidrule(lr){1-11}
    \multirow{3}{*}{Random Replay}
      & best@1  & $\times$ & 80.14 & 187.71 & \textbf{169.45} & 110.71 & $\times$ & 332.25 & $\times$ & 2.06 \\
      & best@4  & 71.20 & 80.14 & 187.16 & \textbf{169.45} & 88.46 & 133.08 & \textbf{326.29} & 755.16 & 2.13 \\
      & best@16 & 71.11 & 80.12 & 187.16 & 168.29 & 88.46 & 132.84 & 325.73 & 729.88 & 2.75 \\
    \cmidrule(lr){1-11}
    \multirow{3}{*}{PCPO}
      & best@1  & \textbf{70.80} & \textbf{80.12} & 187.82 & 171.40 & \textbf{88.59} & $\times$ & \textbf{328.14} & 755.41 & \textbf{1.75} \\
      & best@4  & \textbf{70.36} & \textbf{78.83} & \textbf{186.40} & 171.40 & \textbf{88.06} & 132.93 & 326.58 & 748.17 & \textbf{1.63} \\
      & best@16 & \textbf{70.15} & \textbf{78.83} & \textbf{186.25} & \textbf{168.26} & 87.65 & \textbf{132.28} & \textbf{303.11} & \textbf{727.46} & \textbf{1.13} \\
    \bottomrule
  \end{tabular}
  \end{adjustbox}

  \vspace{2mm}

  \begin{adjustbox}{max width=\textwidth}
  \begin{tabular}{llccccccccc}
    \toprule
    Method & Metric
      & superblue1 & superblue3 & superblue4 & superblue5
      & superblue7 & superblue10 & superblue16 & superblue18 & Avg Rank \\
    \midrule
    \multirow{3}{*}{RePO}
      & best@1  & 398.20 & $\times$ & 293.76 & 459.24 & $\times$ & 865.78 & \textbf{401.21} & $\times$ & 2.50 \\
      & best@4  & 389.03 & 465.33 & 292.03 & 452.65 & 602.12 & 863.00 & 401.18 & 224.32 & 2.50 \\
      & best@16 & 387.76 & 460.43 & 290.29 & 451.82 & 549.79 & 858.22 & 400.94 & 223.96 & 2.63 \\
    \cmidrule(lr){1-11}
    \multirow{3}{*}{Random Replay}
      & best@1  & $\times$ & \textbf{458.95} & 292.13 & $\times$ & 613.61 & 861.91 & 401.64 & 231.89 & 2.25 \\
      & best@4  & 389.08 & \textbf{458.95} & 292.13 & \textbf{450.05} & 557.06 & 861.91 & 401.64 & 224.35 & 2.25 \\
      & best@16 & 388.24 & 458.94 & 291.56 & \textbf{450.05} & 553.28 & 857.01 & 400.70 & 223.40 & 2.25 \\
    \cmidrule(lr){1-11}
    \multirow{3}{*}{PCPO}
      & best@1  & \textbf{389.67} & 463.14 & \textbf{291.58} & \textbf{457.61} & \textbf{551.07} & \textbf{860.22} & 401.59 & \textbf{223.47} & \textbf{1.25} \\
      & best@4  & \textbf{388.95} & 461.36 & \textbf{291.58} & 452.20 & \textbf{549.45} & \textbf{860.22} & \textbf{400.73} & \textbf{223.47} & \textbf{1.25} \\
      & best@16 & \textbf{385.68} & \textbf{455.29} & \textbf{289.79} & 450.76 & \textbf{547.43} & \textbf{856.73} & \textbf{400.31} & \textbf{222.82} & \textbf{1.13} \\
    \bottomrule
  \end{tabular}
  \end{adjustbox}
\end{table*}

\begin{figure}[t]
\centering
\centering
\includegraphics[width=\textwidth]{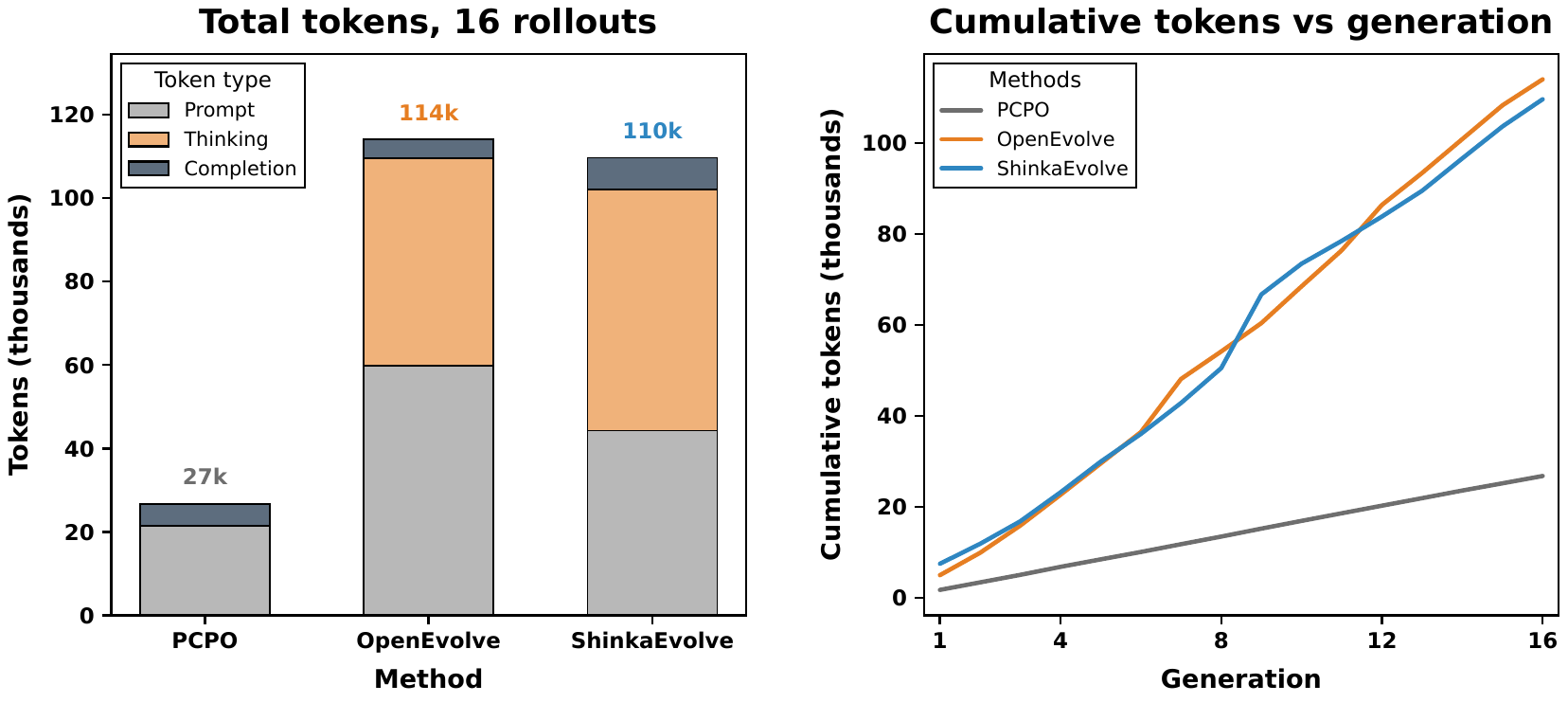} 
\caption{Inference-time token cost statistics of PCPO and two evolutionary methods.}
\label{fig:token}
\end{figure}

\subsubsection{Ablation}\label{appendix-gp-ablation}
We first conduct ablation experiments on the Regularized Off-policy Update and the Refreshing On-policy Update in the Hybrid Policy Update module of PCPO. GRPO~\citep{deepseek-r1} can be seen as the RL method that only adopts the Refreshing On-policy Update, and we also test the performance of PCPO without the Refreshing On-policy Update, named off-policy GRPO. The training result comparison is shown in Figure~\ref{fig:ablation}. We can observe that the training performance degradation occurs in GRPO and off-policy GRPO, which provides evidence for the effectiveness of the Regularized Off-policy Update that utilizes elite experiences for sample-efficient training and the effectiveness of the Refreshing On-policy Update that enhances exploration. We also show the sample efficiency of the mixed-case training process of PCPO compared to GRPO in the form of best-so-far normalized reward curves in Figure~\ref{fig:mix_training}, intuitively demonstrating that the Regularized Off-policy Update efficiently utilizes the population experiences for better overall performance. 

We also conduct experiments to show the effectiveness of the Global Population module of PCPO. As shown in Table~\ref{tab:repo}, PCPO outperforms RePO~\citep{li2025repo}, an RL method that makes the hybrid objective explicit, adding an off-policy GRPO term over historical outputs retrieved per prompt, which increases the number of effective optimization steps when on-policy groups collapse to zero advantage. Here we equip RePO with the reward-oriented replay strategy introduced in its original paper, which achieves the best performance on Qwen3 in the experiments of its original paper. RePO shows certain stability and competitive results on a small budget of rollouts (i.e., best@1 and best@4) in some cases due to its hybrid update scheme similar to our PCPO. However, RePO only selects best-performing samples for off-policy samples, which easily gets stuck in local optima, underperforming on the best@16 metric. PCPO also outperforms its variant with random replay, where randomly selected trajectories are sent for off-policy updates. These results stress the importance of Global Population, which balances the quality and diversity of the off-policy samples, selecting the most promising candidates.

\subsubsection{Reducing token costs} \label{appendix-gp-token}
As shown in Figure~\ref{fig:token}, PCPO reduces inference-time token cost by internalizing grounded domain knowledge and distilling the detailed prompt (e.g., in-context exemplars that evolution-time methods repeatedly inject). After training, the policy can be decoded without lengthy thinking, so the model no longer spends many tokens on a long chain of hallucinated rationales that degrade the final program, saving lots of tokens on the thinking part. The prompt required is also much shorter: the task specification and in-context exemplars that evolution-time methods repeatedly inject have already been distilled into the weights, yielding an effect analogous to context distillation~\citep{context-distillation,learning-by-distilling-context}. Empirically, this shrinks the per-rollout budget from a growing prompt-plus-thinking trace to a compact instruction and a short code completion, cutting total tokens by about $4\times$ relative to OpenEvolve and ShinkaEvolve over the same 16 rollouts, while the generated programs demonstrate higher quality. Therefore, the cost of training is worthwhile, because after a single training run, PCPO can achieve outstanding performance in specific domains with minimal token consumption. This stems from PCPO internalizing domain‑specific knowledge into its own weights, rather than repeatedly injecting external contexts (e.g., in the form of skills) or conducting a long chain of thinking.

\section{GPU Kernel Design}\label{appendix-kd}
\subsection{Introduction to GPU Kernel Design}
GPU kernel design is a hardware-aware optimization process that maps high-level tensor computations onto the parallel execution and memory hierarchy of a target accelerator~\citep{kernelbench}. An effective kernel must identify computationally dominant operators and employ techniques such as operator fusion, tiling, coalesced memory access, and specialized instructions, including tensor-core operations, to reduce kernel-launch overhead, improve data locality, and increase hardware utilization~\citep{tillet2019triton, spector2025thunderkittens}. Because the optimal implementation depends strongly on tensor shapes, numerical precision, and device-specific characteristics, kernel design requires careful trade-offs among computational throughput, memory bandwidth, resource occupancy, and functional correctness. Moreover, compiler diagnostics, correctness tests, execution statistics, and profiling measurements provide essential feedback for iterative refinement, enabling developers to eliminate runtime failures and progressively improve performance. Consequently, GPU kernel design should be regarded not merely as low-level code generation, but as a joint optimization problem spanning algorithmic transformation, parallel program construction, and hardware-specific performance tuning. 

\subsection{Experiment Settings}
We compute the speedup metric with respect to the PyTorch Eager baseline on NVIDIA RTX 4090 GPUs. We follow KernelBench and define speedup as the ratio of PyTorch Eager wall-clock time to the generated kernel time, $\text{Speedup}=T_\text{eager}/T_\text{kernel}$. Correctness is checked against the reference Model on randomized inputs. Kernel runtime is measured online with the KernelBench timing protocol (CUDA events, warmup, repeated trials; we report the mean). For all methods, the budget for one case is 64 rollouts, including the RL methods, testing their inference-time learning ability. OpenEvolve, ShinkaEvolve, GRPO, and PCPO use the Qwen3-8B as the base model. Codex uses GPT-5.5 as the base model.

Note that in our setting, the Eager reference $T_\text{eager}$ is not re-timed at evaluation time. We profile each task once on the evaluation GPU under the same software stack, precision, and input shapes as the later kernel measurements, using the official Eager baseline procedure, and freeze the resulting per-task mean. All subsequent speedups use this snapshot. We freeze the baseline for two reasons. First, GPU wall-clock measurements fluctuate with thermal state, clocking, cache occupancy, and residual CUDA context. Since speedup is a ratio, independent noise in $T_\text{eager}$ is amplified and can change the real speedup of a kernel to a large extent. Second, live Eager timing is easily contaminated by the evaluation process itself (custom-kernel compilation, device context, and multi-job interference)~\citep{kernelbench-verified}. A shared, hardware-local baseline isolates differences in generated kernels, keeps comparisons across methods and seeds on a common scale, and matches KernelBench's official scoring path, which computes speedups against precomputed baseline time tables rather than a freshly timed Eager run.

As shown in Table~\ref{tab:kernelbench-benchmark}, we report the speedup on specific cases rather than the original $\mathrm{fast}_p$ metric used in KernelBench~\citep{kernelbench}. Reporting per-task speedup rather than the aggregate $\mathrm{fast}_p$ score is a deliberate evaluation choice. KernelBench's $\mathrm{fast}_p$ is a thresholded coverage metric: a task contributes $1$ if and only if the kernel is correct and
\begin{equation}
\mathrm{speedup} = \frac{T_{\mathrm{eager}}}{T_{\mathrm{kernel}}} > p.
\end{equation}
This binary reduction has consequences that are unhelpful for diagnosing optimization quality. A barely-above-threshold result (e.g., $1.01\times$) is counted identically to an extreme outlier (e.g., hundreds of $\times$). Prior works have shown that such results are often evaluation artifacts---asynchronous
stream timing~\citep{sun2026cuda}, identity shortcuts on a narrow input distribution, or a weak FP32 Eager denominator~\citep{kernelbench-verified}---rather than physically plausible kernel improvements. We therefore report the continuous speedup $T_{\mathrm{eager}}/T_{\mathrm{kernel}}$
on a handful of Level-1 cases highlighted in the original KernelBench Appendix~D (i.e., MinGPTNewGelu, SoftSign, MatmulWithDiagonalMatrices, and TripletMarginLoss), covering distinct computational structures: transcendental-heavy elementwise computation, bandwidth-dominated pointwise computation, structured broadcasting, and multi-input reduction. Combined with a frozen, hardware-local Eager baseline, this protocol is intended to measure whether a method actually improves these kernels, not whether it can accumulate $\mathrm{fast}_p$ credit by crossing a threshold on many tasks.

\section{Cross‑Model Experiment}\label{appendix-cross-model}
\begin{figure}[t]
  \centering
  \begin{subfigure}[b]{0.48\textwidth}
    \centering
    \includegraphics[width=0.8\textwidth]{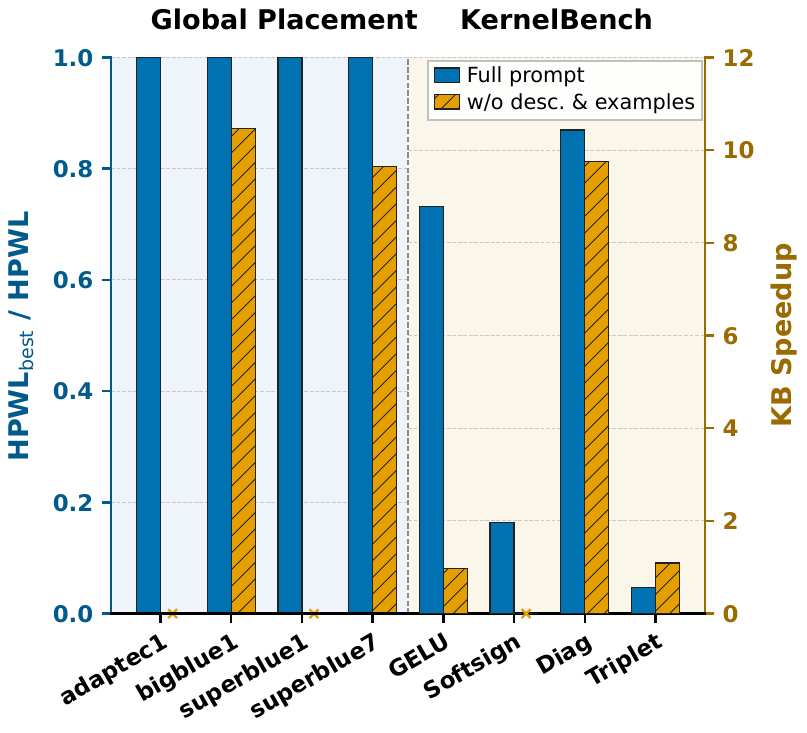}
    \caption{An illustration of relying on external domain knowledge rather than intrinsic knowledge.}
    \label{ood-1}
  \end{subfigure}
  \hfill
  \begin{subfigure}[b]{0.48\textwidth}
    \centering
    \includegraphics[width=0.8\textwidth]{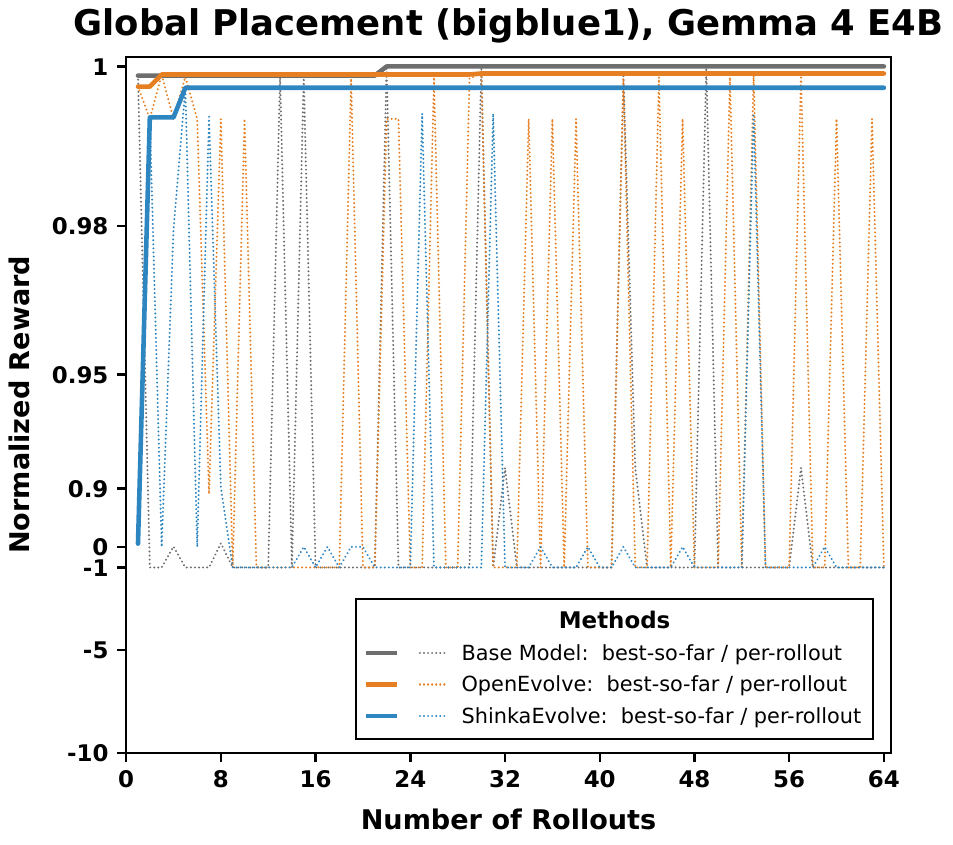}
    \caption{An illustration of in-context evolutionary stagnation of OpenEvolve and ShinkaEvolve.}
    \label{stagnation-1}
  \end{subfigure}
\caption{Diagnostic experiments based on Gemma4-E4B.}
\end{figure}

\begin{table*}[t]
  \centering
  \caption{Best@$n$ HPWL ($\times 10^6$) of Gemma~4 E4B-it on 16 global placement cases. Lower is better.}
  \label{tab:gemma4-e4b-globalplacement}
  \small
  \setlength{\tabcolsep}{3pt}
  \begin{adjustbox}{max width=\textwidth}
  \begin{tabular}{lcccccccccccc}
    \toprule
     & \multicolumn{3}{c}{Gemma~4-E4B} & \multicolumn{3}{c}{OpenEvolve} & \multicolumn{3}{c}{ShinkaEvolve} & \multicolumn{3}{c}{PCPO} \\
    \cmidrule(lr){2-4} \cmidrule(lr){5-7} \cmidrule(lr){8-10} \cmidrule(lr){11-13}
     & best@1 & best@4 & best@16 & best@1 & best@4 & best@16 & best@1 & best@4 & best@16 & best@1 & best@4 & best@16 \\
    \midrule
    \underline{adaptec1}   & $\times$ & $\times$ & $\times$ & $\times$ & 83.42 & 83.41 & $\times$ & \textbf{71.05} & 71.05 & $\times$ & 71.10 & \textbf{70.97} \\
    adaptec2               & \textbf{88.34} & 88.34 & 85.71 & $\times$ & $\times$ & 478.26 & $\times$ & $\times$ & $\times$ & $\times$ & \textbf{80.52} & \textbf{80.05} \\
    adaptec3               & $\times$ & 233.91 & 233.91 & 887.06 & 206.11 & 198.11 & $\times$ & $\times$ & $\times$ & \textbf{193.19} & \textbf{188.17} & \textbf{187.45} \\
    adaptec4               & $\times$ & $\times$ & \textbf{170.05} & $\times$ & 174.02 & 173.97 & \textbf{172.92} & 172.92 & 172.92 & 176.75 & \textbf{172.82} & 171.48 \\
    \underline{bigblue1}   & 88.68 & 88.68 & 88.68 & 89.66 & 88.59 & 88.59 & 272.00 & 92.53 & 89.77 & \textbf{88.00} & \textbf{87.89} & \textbf{87.63} \\
    bigblue2               & $\times$ & 133.15 & \textbf{132.39} & 133.00 & 133.00 & 132.98 & $\times$ & 132.92 & 132.92 & \textbf{132.79} & \textbf{132.65} & 132.65 \\
    bigblue3               & $\times$ & $\times$ & 320.29 & $\times$ & $\times$ & $\times$ & $\times$ & $\times$ & 981.92 & \textbf{313.78} & \textbf{313.78} & \textbf{313.78} \\
    bigblue4               & $\times$ & $\times$ & 757.63 & $\times$ & \textbf{757.54} & 757.48 & $\times$ & $\times$ & 757.77 & $\times$ & 758.72 & \textbf{731.61} \\
    \underline{superblue1} & $\times$ & \textbf{389.07} & \textbf{389.07} & 399.11 & 399.11 & 391.88 & $\times$ & 979.70 & 468.06 & \textbf{389.82} & 389.82 & 389.82 \\
    superblue3             & $\times$ & $\times$ & 571.76 & $\times$ & $\times$ & $\times$ & $\times$ & 475.34 & 458.20 & \textbf{677.41} & \textbf{465.86} & \textbf{446.17} \\
    superblue4             & \textbf{291.63} & \textbf{291.63} & \textbf{291.63} & 314.93 & 292.07 & 291.93 & 294.97 & 294.97 & 292.09 & 416.15 & 296.19 & 291.67 \\
    superblue5             & $\times$ & $\times$ & 456.26 & $\times$ & $\times$ & 489.50 & 459.12 & 454.87 & 454.87 & \textbf{454.78} & \textbf{454.78} & \textbf{451.74} \\
    \underline{superblue7} & $\times$ & \textbf{552.19} & \textbf{549.70} & $\times$ & 565.25 & 565.25 & \textbf{575.88} & 575.88 & 575.88 & 651.35 & 557.89 & 557.88 \\
    superblue10            & $\times$ & $\times$ & 861.83 & $\times$ & 963.61 & 876.90 & $\times$ & \textbf{864.01} & 864.01 & $\times$ & 889.10 & \textbf{861.22} \\
    superblue16            & $\times$ & $\times$ & 402.86 & \textbf{401.35} & \textbf{401.35} & \textbf{401.35} & $\times$ & $\times$ & $\times$ & 492.88 & 402.33 & 402.33 \\
    superblue18            & $\times$ & $\times$ & 228.14 & \textbf{223.84} & \textbf{223.84} & \textbf{223.83} & 839.26 & 223.88 & 223.88 & 225.25 & 225.25 & 224.38 \\
    \midrule
    Avg rank              & 2.96 & 3.00 & 2.31 & 2.58 & 2.53 & 2.94 & 2.77 & 2.78 & 3.25 & \textbf{1.69} & \textbf{1.69} & \textbf{1.50} \\
    \bottomrule
  \end{tabular}
  \end{adjustbox}
\end{table*}

To provide results that are consistent with the hypotheses in the main paper in a broader sense, we conduct experiments on different base models (i.e., Gemma4-E4B~\citep{team2026gemma}, apart from Qwen3-8B~\citep{yang2025qwen3}). Figure~\ref{ood-1} shows that Gemma4-E4B relies on external domain knowledge rather than intrinsic knowledge to improve its performance. Figure~\ref{stagnation-1} shows that in-context evolutionary stagnation also occurs with the base model Gemma4-E4B, which provides evidence that it is a common phenomenon.

Table~\ref{tab:gemma4-e4b-globalplacement} shows the result comparison of the base model Gemma4-E4B, two evolutionary methods, and PCPO based on Gemma4-E4B. The results align with Table~\ref{tab:globalplacement-benchmark} in the main paper with the base model Qwen3-8B: PCPO achieves the best average rank on all best@$n$ metrics, demonstrating its improvement within a small rollout budget and its stability on one-shot generations. In-context evolutionary methods still show worse average performance compared to the base model, which provides evidence that this is a common phenomenon, as analysed in Section~\ref{sec:in-context stagnation}.

\end{document}